\documentclass[twoside,11pt]{article} 
\usepackage{jmlr2e} 
\usepackage{amsmath,amsfonts,stmaryrd,amssymb,enumitem}
\usepackage{booktabs,todonotes}

\usepackage{algpseudocode}
\algnewcommand\algorithmicinput{\textbf{Input:}}
\algnewcommand\INPUT{\item[\algorithmicinput]}

\algnewcommand\algorithmicoutput{\textbf{Output:}}
\algnewcommand\OUTPUT{\item[\algorithmicoutput]}

\usepackage{pifont}

\usepackage{tikz}
\usetikzlibrary{arrows}
\usetikzlibrary{calc}
\usetikzlibrary{shapes}
\usetikzlibrary{fit}
\usetikzlibrary{matrix} 
\usetikzlibrary{positioning}
\usetikzlibrary{decorations.pathreplacing}

\renewcommand{\eqref}[1]{Eq.~(\ref{eq:#1})}
\newcommand{\figref}[1]{Figure~\ref{fig:#1}}
\newcommand{\defref}[1]{Def.~\ref{def:#1}}
\newcommand{\tabref}[1]{Table~\ref{tab:#1}}        
\newcommand{\secref}[1]{Section~\ref{sec:#1}}
\newcommand{\thmref}[1]{Theorem~\ref{thm:#1}}
\newcommand{\lemref}[1]{Lemma~\ref{lem:#1}}

\newcommand{\appref}[1]{Appendix~\ref{ap:#1}}

\newcommand{\myalgref}[1]{Alg.~\ref{alg:#1}}

\renewcommand{\P}{\mathbb{P}}
\newcommand{\E}{\mathbb{E}}

\newcommand{\dotprod}[1]{\langle #1 \rangle}

\DeclareMathOperator*{\argmax}{argmax}

\newcommand{\cB}{\mathcal{B}}

\newcommand{\cD}{\mathcal{D}}
\newcommand{\cE}{\mathcal{E}}
\newcommand{\cF}{\mathcal{F}}
\newcommand{\cG}{\mathcal{G}}

\newcommand{\cS}{\mathcal{S}}

\newcommand{\cV}{\mathcal{V}}

\usepackage[utf8]{inputenc} 
\usepackage[T1]{fontenc}    
\usepackage{lmodern}
\usepackage{url}            
\usepackage{booktabs}       
\usepackage{amsfonts}       
\usepackage{microtype}      
\usepackage{enumitem}
\usepackage{graphicx}
\usepackage{bm}
\usepackage{bbm}
\usepackage{verbatim}
\usepackage[ruled,vlined]{algorithm2e}
\usepackage{tikz}
\usepackage{todonotes}
\usepackage{wasysym}
\usepackage{multirow}
\usepackage{adjustbox}
\usepackage{array}

\usepackage{afterpage}

\makeatletter
\def\@cline#1-#2\@nil{%
	\omit
	\@multicnt#1%
	\advance\@multispan\m@ne
	\ifnum\@multicnt=\@ne\@firstofone{&\omit}\fi
	\@multicnt#2%
	\advance\@multicnt-#1%
	\advance\@multispan\@ne
	\leaders\hrule\@height\arrayrulewidth\hfill
	\cr
	\noalign{\nobreak\vskip-\arrayrulewidth}}
\makeatother

\newcommand{\uni}[1]{\mathbf{U}(#1)}
\newcommand{\algname}{$\mathrm{ActiveBNSL}$}
\newcommand{\score}{\cS}
\newcommand{\escore}{\hat{\cS}}
\newcommand{\family}[1]{\langle #1 \rangle}
\newcommand{\event}{\eta}
\newcommand{\subsets}[1]{\llbracket #1 \rrbracket ^{k+1}}
\newcommand{\ecset}{\cE}
\newcommand{\accept}{\mathrm{Acc}}
\newcommand{\actf}{\mathrm{Cand}}
\newcommand{\cons}[2]{#1^{\cap {#2}}}
\newcommand{\struct}{\cG}

\newcommand{\xa}{X_a}
\newcommand{\xb}{X_b}
\newcommand{\Pa}{\Pi_a}
\newcommand{\Pb}{\Pi_b}
\newcommand{\codeurl}{\url{https://github.com/noabdavid/activeBNSL}}

\newenvironment{claim}{\paragraph{Claim:}}{}
\newenvironment{claimproof}{\par\noindent{\bf Proof of claim:\ }}{\hfill$\Box$\\[2mm]}

\makeatletter
\newcommand{\manuallabel}[2]{\def\@currentlabel{#2}\label{#1}}
\makeatother

\newcommand{\papertitle}{Active Structure Learning of Bayesian Networks\\in an Observational Setting}
\title{\papertitle}

\begin{document}
	
	\title{\papertitle}
	
	\author{\name Noa Ben-David \email bendanoa@post.bgu.ac.il \\
		\name Sivan Sabato \email sabatos@cs.bgu.ac.il\\
		\addr Department of Computer Science\\
		Ben-Gurion Univesity of the Negev\\
		Beer Sheva 8410501, Israel.}

	\maketitle
	
	\begin{abstract} 
		We study active structure learning of Bayesian networks in an observational setting, in
		which there are external limitations on the number of variable values that can be observed from the same sample. Random samples are drawn from the joint distribution of the network
		variables, and the algorithm iteratively selects which variables to observe in the next
		sample.
		We propose a new active learning algorithm for this setting, that finds with a high probability a structure with a score that is $\epsilon$-close to the optimal score. We show that for a class of distributions that we term \emph{stable}, a sample complexity reduction of up to a factor of $\widetilde{\Omega}(d^3)$ can be obtained, where $d$ is the number of network variables. We further show that in the worst case, the sample complexity of the active algorithm is guaranteed to be almost the same as that of a naive baseline algorithm. To supplement the theoretical results, we report experiments that compare the performance of the new active algorithm to the naive baseline and demonstrate the sample complexity improvements. Code for the algorithm and for the experiments is provided at \codeurl.

	\end{abstract}
	
	
	\section{Introduction}
	\label{sec:intro}
	
	In this work, we study active structure learning of Bayesian networks in an observational (that is, a non-interventional) setting, in
	which there are external limitations on the number of variable values that can be observed from the same sample. Bayesian networks are a popular modeling tool used in various applications, such as medical analysis \citep{HaddawyHaKaLaSaKaSi18, XuThAnLiPaNa18}, human activity models \citep{LiuWaHuQiWeRo18} and engineering \citep{RovinelliSaPrGuLeLu18,CaiKoLiLiYuXuJi19}, among others. A Bayesian network is a graphical model that encodes probabilistic relationships among random variables. 
	The structure of a Bayesian network is represented by a Directed Acyclic Graph (DAG) whose nodes are the random variables, where the edge structure represents statistical dependency relations between the variables.
	Structure learning of a Bayesian network \cite[see, e.g.,][]{KollerFr09} aims to find the structure (DAG) that best matches the joint distribution of given random variables, using i.i.d.~samples of the corresponding random vector. Structure learning is used, for instance, for uncovering gene interaction \citep{FriedmanLiNaPe00}, cancer prediction \citep{WitteveenNaVlSiIj18}, and fault diagnosis \citep{CaiHuXi17}.

	In the setting that we study, a limited number of variable measurements can be
	observed from each random sample. The algorithm iteratively selects which of
	the variable values to observe from the next random sample.  This is of
	interest, for instance, when the random vectors represent
	patients participating in a study, and each measured variable corresponds to
	the results of some medical test. Patients would not agree to undergo many
	different tests, thus the number of tests taken from each patient is
	restricted. Other examples include obtaining a synchronized measurement from
	all sensors in a sensor network with a limited bandwidth
	\citep{DasarathySiAa16}, and recording neural activity with a limited
	recording capacity \citep{SoudryKeStOhIyPa13,TuragaBuPaDaPeHaMa13}. The goal
	is to find a good structure based on the the restricted observations, using a small number of random samples. Learning with limited observations from each
	sample was first studied in \cite{BendavidDi98} for binary classification, and
	has thereafter been studied for other settings, such as online learning
	\citep{CesaBianchiShSh11, HazanKo12} and linear regression
	\citep{KuklianskySh15}. This work is the first to study limited observations in the context of structure learning of Bayesian networks.

	We study the active structure-learning problem in a score-based framework \citep{CooperHe92,HeckermanGeCh95}, in which each possible DAG has a score that depends on the unknown distribution, and the goal is to find a DAG with a near-optimal score.
	
	We propose a new active learning algorithm, \algname,  that finds with a high probability a structure with a score that is $\epsilon$-close to the optimal structure. We compare the sample complexity of our algorithm to that of a naive algorithm, which  observes every possible subset of the variables the same number of times. We show that the improvement in the sample complexity can be as large as a factor of $\widetilde{\Omega}(d^3)$, where $d$ is the dimension of the random vector. This improvement is obtained for a class of distributions that we define, which we term \emph{stable} distributions. We further show that in the worst case, the sample complexity of the active algorithm is guaranteed to be almost the same as that of the naive algorithm.
	
	A main challenge in reducing the number of samples in this setting  is the fact that each structure usually has multiple equivalent structures with the same quality score. \algname\ overcomes this issue by taking equivalence classes into account directly. An additional challenge that we address is characterizing distributions in which a significant reduction in the sample complexity is possible. Our definition of stable distributions captures such a family of distributions, in which there is a large number of variables with relationships that are easy to identify, and a small number of variables whose relationships with other variables are harder to identify. 
	
	Finding the optimal structure for a Bayesian network, even in the fully-observed setting, is computationally hard \citep{Chickering96}. Nonetheless, algorithms based on relaxed linear programming are successfully used in practice \citep{Cussens11,BartlettCu13,BartlettCu17}. \algname\ uses existing structure-search procedures as black boxes. This makes it easy to implement a practical version of \algname\ using existing software packages. We implement \algname\ and the naive baseline, and report experimental results that demonstrate the sample complexity reduction on stable distributions. The code for the algorithms and for the experiments is provided in \codeurl.

	\paragraph{Paper structure}
	We discuss related work in \secref{related}. The setting and notation are defined in \secref{prel}. In \secref{naive}, a naive baseline algorithm is presented and analyzed. The proposed active algorithm, \algname, is given in \secref{newalg}. Its correctness is proved in \secref{correct}, and sample complexity analysis is given in \secref{sample}, where we show that significant savings can be obtained for the class of stable distributions that we define. In \secref{stableex}, we describe an explicit class of stable distributions. In \secref{experiments}, we report experiments that compare the performance of our implementation of \algname\ to that of the naive baseline, on a family of stable distributions. We conclude with a discussion in \secref{discussion}. Some technical proofs are deferred to the appendices. 
	
	\section{Related work}\label{sec:related}

	Several general approaches exist for structure learning of Bayesian networks.
	In the constraint-based approach \citep{Meek95,SpirtesGlSc00}, pairwise statistical tests are performed to reveal conditional independence of pairs of random variables. This approach is asymptotically optimal, but less successful on finite samples \citep{HeckermanMeCo06}, unless additional realizability assumptions are made. When restricting the in-degree of the DAG, the constraint-based approach is computationally efficient. However, the computational complexity and the sample complexity required for ensuring an accurate solution depend on the difficulty of identifying the conditional independence constraints \citep{CanonneDiKaSt18} and cannot be bounded independently of that difficulty. The \emph{score-based} approach assigns a data-dependent score to every possible structure, and searches for the structure that maximizes this score. One of the most popular score functions for discrete distributions is the Bayesian Dirichlet (BD) score \citep{CooperHe92}. Several flavors of this score have been proposed \citep{HeckermanGeCh95,Buntine91}. The BD scores are well-studied \citep{Chickering02,TsamardinosBrAl06,CamposJi11,Cussens11,BartlettCu13,BartlettCu17} and widely used in applications \cite[see, e.g.,][]{MariniTrBaSaDiMaMaCoBe15,LiChZhNg16,HuKe18}.

	Finding a Bayesian network with a maximal BD score is computationally hard under standard assumptions \citep{Chickering96}. However, successful heuristic algorithms have been suggested. Earlier algorithms such as \cite{Chickering02,TsamardinosBrAl06} integrate greedy hill-climbing with other methods to try and escape local maxima. Other approaches include dynamic programming \citep{SilanderMy12} and integer linear programming \citep{Cussens11,BartlettCu13,BartlettCu17}.

	Previous works on active structure learning of Bayesian networks \citep{TongKo01,HeGe08,LiLe09,SquiresMaGrKaKoSh20} assume a causal structure between the random variables, and consider an interventional (or experimental) environment, in which the active learner can set the values of some of the variables and then measure the values of the others. This is crucially different from our model, in which there is no causality assumption and the interaction is only by observing variables, not setting their values. 
	
	Active learning for undirected graphical models has been studied both in a full observational setting \citep{VatsNoBa14, DasarathySiAa16}, where all variable values are available in each sample, and in a setting of limited observations \citep{ScarlettCe17, Dasarathy19, VinciDaGe19}.  
	
	More generally, interactive unsupervised learning has been studied in a variety of contexts, including clustering \citep{AwasthiBaVo17} and compressive sensing \citep{HauptCaNo09,MalloyNo14}.

	\section{Setting and Notation}\label{sec:prel}
	
	For an integer $i$, denote $[i] := \{1,\ldots,i\}$.
	A Bayesian network models the probabilistic relationships between random variables  $\mathbf{X} = (X_1,\dots,X_d)$. It is represented by a DAG $G=(V,E)$ and a set of parameters $\Theta$. $V=[d]$ is the set of nodes, where node $i$ represents the random variable $X_i$. $E$ is the set of directed edges. We denote the set of parents of a variable $X_i\in \mathbf{X}$ under $G$ by $\Pi_i(G) := \{X_j \mid (j,i)\in E\}$. We omit the argument $G$ when it is clear from context. $G$ and $\Theta$ 
	together define a joint distribution over $\mathbf{X}$ which we denote by $P_{G,\Theta}(\mathbf{X})$. In the case where the distributions of all $X_i$ are discrete, $\mathbf{X} := (X_1,\ldots,X_d)$ is a multivariate multinomial random vector and $\Theta$ specifies the multinomial distributions $P(X_i \mid \Pi_i)$. 
	In this case, the joint distribution defined by $G$ and $\Theta$ is given by
	\[
	P_{G,\Theta}(\mathbf{X})=\prod_{i\in [d]}P(X_i \mid \Pi_i(G), \Theta).
	\]

	We call any possible pair of a child variable and parent set over $[d]$ a \emph{family}, and denote it by $\family{X_i, \Pi}$, where $X_i$ is the child random variable and $\Pi \subseteq \{X_1,\ldots, X_d\}\setminus \{X_i\}$ defines the set of parents assigned to $X_i$. We refer to a family with a child variable $i$ as \emph{a family of $i$}.
	For convenience, we will use $G$ to denote also the set of families that the structure $G$ induces, so that $f \in G$ if and only if $f = \dotprod{X_i, \Pi_i(G)}$ for some $i \in [d]$.  
	
	Structure learning is the task of finding a graph $G$ with a high-quality fit to the true distribution of $\mathbf{X}$. This is equivalent to finding the family of each of the variables in $[d]$. The quality of the fit of a specific structure $G$ can be measured using the following information-criterion score \citep{Bouckaert95}:
	\[
	\score(G):=-\sum_{i\in [d]}H(X_i \mid \Pi_i(G)) = -\sum_{f \in G}H(f),
	\]
	where $H$ stands for entropy, and $H(f)$, for a family $f = \family{X_i, \Pi}$, is defined as $H(f) := H(X_i \mid \Pi)$. 
	As shown in \cite{Bouckaert95}, for discrete distributions, maximizing the plug-in estimate of $\score(G)$ from data, or small variants of it, is essentially equivalent to maximizing some flavors of the BD score. 
	
	For both statistical and computational reasons, in structure learning one commonly searches for a graph in which each variable has at most $k$ parents, for some small integer constant $k$ \cite[see, e.g.,][]{Heckerman98}. Denote by $\struct_{d,k}$ the set of DAGs over $d$ random variables such that each variable has at most $k$ parents. Denote the set of all possible families over $[d]$ with at most $k$ parents  by $\cF_{d,k}$. 
	Denote by $\score^* := \max_{G \in \struct_{d,k}} \score(G)$ the optimal score of a graph in $\struct_{d,k}$ for the true distribution of $\mathbf{X}$. Note that $\mathbf{X}$ can have any joint distribution, not necessarily one of the form $P_{G,\Theta}(\mathbf{X})$ for some $G\in\struct_{d,k}$ and $\Theta$. In cases where $P(\mathbf{X})$ is in fact of the latter form, it has been shown for Gaussian and multinomial random variables \citep{Chickering02}, that for any graph $G$ that maximizes $\score(G)$, there is some $\Theta$ such that the distribution of $\mathbf{X}$ is equal to $P_{G,\Theta}(\mathbf{X})$. In other words, whenever there exists some graph which exactly describes $P(\mathbf{X})$, score maximization will find such a graph. Our goal is to find a structure $\hat{G} \in \struct_{d,k}$ such that with a probability of at least $1-\delta$, $\score(\hat{G}) \geq \score^* - \epsilon$. Denote by $\struct^*$ the set of all optimal structures, $\struct^*:= \{G\in\mathcal{G}_{d,k}\mid\score(G)=\score^*\}.$
	
	We study an active setting in which the algorithm iteratively selects a subset of variables to observe, and then draws an independent random copy of the random vector $\mathbf{X}$. In the sampled vector, only the values of variables in the selected subset are revealed. The goal is to find a good structure using a small number of such random vector draws. 
	We study the case where the maximal number of variables which can be revealed in any vector sample is $k+1$. This is the smallest number that allows observing the values of a variable and a set of potential parents in the same random sample.    
	We leave for future work the generalization of our approach to other observation sizes. For a set $A$, denote by $\subsets{A}$  the set of subsets of $A$ of size $k+1$. For an integer $i$, $\subsets{i}$ is used as a shorthand for $\subsets{[i]}$.

	\section{A Naive Algorithm}
	\label{sec:naive}
	We first discuss a naive approach, in which all variable subsets of size $k+1$ are observed the same number of times, and these observations are used to estimate the score of each candidate structure. 
	As an estimator for $H(f)$, one can use various options, possibly depending on properties of the distribution of $\mathbf{X}$ \cite[see, e.g.,][]{Paninski03}. Denoting this estimator by $\hat{H}$, 
	the empirical score of graph $G$ is defined as 
	\begin{equation}\label{eq:empscore}
	\escore(G):=-\sum_{f \in G}\hat{H}(f) = -\sum_{i\in [d]}\hat{H}(X_i \mid \Pi_i(G)).
	\end{equation}
	For concreteness, assume that $\{X_i\}_{i \in [d]}$ are discrete random variables with a bounded support size. In this case, $\hat{H}$ can be set to the plug-in estimator, obtained by plugging in the empirical joint distribution of $\{X_i\} \cup \Pi$ into the definition of conditional entropy, where the empirical joint distribution is based on samples in which all of these variables were observed together. 
	
	For $\epsilon >0, \delta \in (0,1)$, let $N(\epsilon,\delta)$ be a \emph{sample-complexity upper bound} for $\hat{H}$. $N$ is a function such that for any fixed $f = \family{i,\Pi}$, if $\hat{H}(f)$ is calculated using at least $N(\epsilon,\delta)$ i.i.d.~copies of the vector $\mathbf{X}$ in which the values of $X_i$ and $\Pi$ are all observed, then with a probability of at least $1-\delta$, $|\hat{H}(f)-H(f)| \leq \epsilon$. For instance, in the case of discrete random variables with the plug-in estimator, the following lemma, proved in \appref{N},  shows that one can set $N(\epsilon,\delta) = \widetilde{\Theta}(\log(1/\delta)/\epsilon^2)$. Interestingly, this bound is the same order as the bound for the unconditional entropy \citep{AntosKo01}.

	\begin{lemma}\label{lem:N}
		Let $\delta\in(0,1)$ and $\epsilon>0$. Let $A,B$ be discrete random variables with support cardinalities $M_a$ and $M_b$, respectively. Let $\hat{H}(A \mid B)$ be the plug-in estimator of $H(A \mid B)$, based on $N$ i.i.d.~copies of $(A,B)$, where 
		\[
		N \geq \max\left\{\frac{2}{\epsilon^2}\log(\frac{2}{\delta})\cdot \log^2(2\log(2/\delta)/\epsilon^2), e^2, M_a, (M_a-1)M_b/\epsilon\right\}.
		\]
		Then 
		$P(|\hat{H}(A \mid B)-H(A \mid B)|>\epsilon)\leq \delta.$
		
	\end{lemma}

	Now, consider a naive algorithm which operates as follows:
	\begin{enumerate}
		\item Observe each variable subset in $\subsets{d}$ in $N(\epsilon/(2d),\delta/|\cF_{d,k}|)$ of the random samples;
		\item Output a structure $G\in \struct_{d,k}$ that maximizes the empirical score defined in \eqref{empscore}, based on the plug-in estimates of $H(f)$ for each family.
	\end{enumerate}
	Applying a union bound over all the families in $\cF_{d,k}$, this guarantees that with a probability at least $1-\delta$, we have that for each family $f \in \cF_{d,k}$, $|\hat{H}(f) - H(f)| \leq \epsilon/(2d)$.
	Therefore, for any graph $G \in \struct_{d,k}$,
	\[
	|\score(G)-\escore(G)|= |-\sum_{f \in G}H(f)+\sum_{f \in G}\hat{H}(f)|
	\leq \sum_{f\in G}|\hat{H}(f)-H(f)|\leq \epsilon/2.
	\]
	Letting $\hat{G} \in \argmax_{G \in \struct_{d,k}}\escore(G)$ and $G^* \in \struct^*$, with a probability at least $1-\delta$,
	\[
	\score(\hat{G})\geq\escore(\hat{G})-\epsilon/2 \geq\escore(G^*)-\epsilon/2 \geq\score(G^*)-\epsilon = \score^*-\epsilon,
	\]
	as required. Note that $|\cF_{d,k}| = d\sum_{i\in [k]}\binom{d-1}{i}\leq d(e(d-1)/k)^k$. Therefore, the sample complexity of the naive algorithm is at most
	\begin{equation}\label{eq:naive}
	|\subsets{d}| \cdot N(\epsilon/(2d),\delta/|\cF_{d,k}|) = \widetilde{\Theta}\left(\frac{kd^{k+3}}{\epsilon^2} \cdot\log(1/\delta)\right),
	\end{equation}
	where the right-hand side is given for the case of a discrete distribution with the plug-in entropy estimator discussed above. 
	Below, we compare the sample complexity of the naive algorithm to the sample complexity of the interactive algorithm that we propose. Therefore, it is necessary to discuss the tightness of the upper bound in \eqref{naive}. Since the naive algorithm samples each subset of size $k+1$ separately, the factor of $|\subsets{d}| = \Theta(d^{k+1})$ cannot be avoided. The dependence on $\epsilon/(2d)$ in $N$ is necessary, as can be observed via simple symmetry arguments. Thus, the only possible sources of looseness in this upper bound are the convergence analysis leading to the function $N(\cdot, \cdot)$, and the use of a union bound which leads to the factor of $\log(|\cF_{d,k}|) = O(k\log(d))$. Since we use the same convergence function $N(\cdot, \cdot)$ also for analyzing the interactive algorithm, and we treat $k$ as a constant, we conclude that comparing the sample complexity of the interactive algorithm to the bound in \eqref{naive} is accurate up to possible logarithmic factors in $d$.

	\section{An Interactive Algorithm: \algname}\label{sec:newalg}
	
	The naive algorithm observes all variable subsets the same number of times. Thus, it does not make use of possible differences in difficulty in identifying different parts of the structure. We propose a new interactive algorithm, \algname, which uses previous observations to identify such differences to reduce the sample complexity. 
	Our approach is based on incrementally adding families to the output structure, until it is fully defined. This reduces the sample complexity in cases where the early acceptance of families allows removing some variables from observation.
	
	A challenging property of Bayesian networks in this respect is \emph{Markov equivalence}. Two graphs are considered Markov equivalent if they induce the same conditional independence constraints \citep{AnderssonMaPe97}. It has been shown that for any dataset, the likelihood of the data given the graph is the same for all Markov equivalent graphs  \citep{HeckermanGeCh95}. Since the score defined in \secref{prel} is the asymptotic log-likelihood of an infinite sample drawn from the distribution, the scores of two Markov equivalent graphs are identical. Almost all DAGs in $\struct_{d,k}$ have other Markov-equivalent DAGs in $\struct_{d,k}$, and usually there is not even a single family that is shared by all Markov-equivalent graphs. \algname\ addresses this issue by directly handling Markov equivalence classes, which are the equivalence classes induced on $\struct_{d,k}$ by the Markov-equivalence relation. Denote the set of equivalence classes (ECs) in $\struct_{d,k}$ by $\ecset_{d,k},$ and note that each equivalence class $E\in\ecset_{d,k}$ is a set of structures, $E \subseteq \struct_{d,k}$. Since all the structures in a given EC have the same score, we can discuss the score of an EC without confusion. For an EC $E \in \ecset_{d,k}$, denote by $\score(E)$ the (identical) score of the graphs in $E$. An \emph{optimal EC} is an EC that maximizes the score over all ECs in $\ecset_{d,k}$. Note that this EC might not be unique.

	\algname, listed in \myalgref{active_algorithm}, gets $d$ (the number of variables) and $k$ (the maximal number of parents) as inputs, in addition to a confidence parameter $\delta$, the required accuracy level $\epsilon$.
	 \algname\ maintains a set $\cV \subseteq [d]$ of the variables whose families have been accepted so far (that is, their parent sets in the output structure have been fixed). The set of accepted families is denoted $\accept_j$, where $j$ denotes an iteration within the algorithm. $\accept_j$ includes exactly one family for each $v \in \cV$. The set of candidate families is denoted $\actf(\cV)$. This is the set of families that have not been precluded so far from participating in the output structure. For simplicity, we set $\actf(\cV)$ to the set of all the families with a child variable not in $\cV$. Note that this does not preclude, for instance, families that would create a cycle when combined with $\accept_j$. As will be evident below, including redundant families in $\actf(\cV)$ does not harm the correctness of \algname.

	\begin{algorithm}[t]
		\DontPrintSemicolon
		\KwIn{Integers $d,k$; $\delta \in (0,1)$; $\epsilon > 0$. }
		\KwOut{A graph $\hat{G}$}
		\nl \textbf{Intialize}: $\accept_1 \leftarrow \emptyset$, $\cV\leftarrow\emptyset$,  $N_0 \leftarrow 0$, $t\leftarrow1$, $j \leftarrow 1$, $T \leftarrow \lceil \log_2(2d) \rceil$, $\epsilon_1\leftarrow\epsilon$.\;\label{init}
		\nl \While{$\epsilon_t > \epsilon/(d-|\cV|)$}{

			\nl $N_t \leftarrow N(\epsilon_t/2, \delta/(T|\cF_{d,k}|))$\label{Nt} \;
			\nl For each subset in $\subsets{d} \setminus \subsets{\cV}$, observe it in $N_t- N_{t-1}$ random samples.\label{sample}\; ;
			\nl \Repeat{$\accept_j = \accept_{j-1}$ } {
				\nl $\theta_j\leftarrow (d-|\cV|)\cdot\epsilon_t$ \;\label{theta_j} \label{iteration}
				
				\nl $\hat{G}_j\leftarrow\argmax\{\escore_t(G) \mid G\in\struct_j\}$ \hspace{2em}\# Recall: $\struct_j \equiv \{ G \in \struct_{d,k} \mid \accept_j \subseteq G\}$\label{maximization}\;
				\nl $\hat{E}_j\leftarrow$ the EC in $\ecset_{d,k}$ that includes $\hat{G}_j$ \label{Gj} \;
				\nl $L_j\leftarrow\{E \in \ecset_{d,k}\mid (\cons{E}{j}\neq\emptyset) \wedge (\escore_t(\cons{\hat{E}_j}{j})-\escore_t(\cons{E}{j})\leq \theta_j) \}$ \label{L_j}\; 
				\nl\eIf{$\exists f\in \uni{\cons{\hat{E}_j}{j}}\cap\actf(\cV)$ such that  $\forall E\in L_j$, $f\in \uni{\cons{E}{j}}$ \label{if}}{
					\nl Set $f \leftarrow \family{X_v,\Pi}$ such that $f$ satisfies the condition  above and $v \in [d]$\label{condition}\;
					\nl $\cV \leftarrow \cV \cup \{v\}$, $\accept_{j+1}\leftarrow\accept_j\cup\{f\}$  \hspace{3em} \# accept family $f$\label{accept}\;
					
				} { \nl $\accept_{j+1}\leftarrow\accept_j$ }
				\nl $j \leftarrow j+1$\;
				
			}
			\nl \textbf{If } $|\cV|=d$ \label{step:if} \textbf{then} set $\hat{G}\leftarrow\accept_j$ and \Return $\hat{G}$. \;
			\nl $t\leftarrow t+1$ \;
			\nl $\epsilon_t \leftarrow \epsilon_{t-1} / 2$ \;
			
		}
		\nl $\epsilon_{\mathrm{last}} \leftarrow \epsilon/(d-|\cV|)$,  $N_T \leftarrow N(\epsilon_{\mathrm{last}}/2, \delta/(T|\actf(\cV)|))$ \label{step:epslast}\;
		\nl For each subset in $\subsets{d} \setminus \subsets{\cV}$, observe it in $N_T- N_{T-1}$ random samples.\;
		
		\nl \textbf{Return} $\hat{G}\leftarrow\argmax\{\escore_T(G) \mid G\in\struct_j\}$. \label{return}\;
		\caption{\algname}\label{alg:active_algorithm}
	\end{algorithm}

	\algname\ works in rounds. At each round $t$, the algorithm observes all the subsets in $\subsets{d}$ that include at least one variable not in $\cV$. Each such subset is observed a sufficient number of times to obtain a required accuracy. Let $\hat{H}_t(f)$ be the estimator for $H(f)$ based on the samples observed until round $t$, and denote the empirical score at round $t$ by $\escore_t(G) :=  -\sum_{f \in G} \hat{H}_t(f)$. Note that unlike the true score, the empirical score of the graphs in an EC might not be the same for all the graphs, due to the use of partial observations. Therefore, we define the empirical score of a set of graphs $A$ as $\escore(A):=\max_{G \in A} \escore(G).$ Iteratively within the round, \algname\ finds some equivalence class $\hat{E}$ that maximizes the empirical score $\escore(\hat{E})$ and is consistent with the families accepted so far. Here, $\struct_j := \{ G \in \struct_{d,k} \mid \accept_j \subseteq G\}$ denotes the set of graphs that are consistent with these families in iteration $j$.  
	We also denote $\uni{A} := \bigcup_{G\in A}G$, the set of all families in any of the graphs in $A$.

	For $E \in \ecset_{d,k}$, denote $\cons{E}{j} := E \cap \struct_j$. \algname\ calculates a threshold $\theta_j$, and searches for a family such that in each of the ECs whose empirical score is $\theta_j$-close to the empirical score of $\hat{E}$, there exists at least one graph which is consistent with $\accept_j$ and includes this family. If such a family exists, it can be guaranteed that it is a part of some optimal structure along with previously accepted families. Therefore, it is accepted, and its child variable is added to $\cV$.
	At the end of every round, the required accuracy level $\epsilon_t$ is halved. Iterations continue until $\epsilon_t\leq\epsilon/(d-|\cV|)$. It is easy to verify that the total number of rounds is at most $T-1 =  \lceil \log_2(d) \rceil$, where $T$ is defined in \algname. In the last round, which occurs outside of the main loop, if any families are still active, the remaining variable subsets are observed based on the required accuracy $\epsilon$. \algname\ then returns a structure which maximizes the empirical score, subject to the constraints set by the families accepted so far.
	
	\begin{table}[t]
		\setlength\extrarowheight{2pt} 
		\centering

		\begin{tabular}{l|l|l|l|l|l}
			$j$	& $\accept_j$ & $\cV$ & $\hat{G}_j$ & $\hat{E}_j$ & $L_j$ \\
			\hline
			1 & $\emptyset$ & $\emptyset$ & $G_3$ & $E_3$ & $\{E_1,E_2,E_3\}$ \\
			\hline
			2 & $\{\family{1,\emptyset}\}$ & $\{1\}$ & $G_3$ & $E_3$ & $\{E_1,E_2,E_3\}$ \\
			\hline
			3 & $\{\family{1,\emptyset},\family{2,\{1\}}\}$ & $\{1,2\}$ & $G_3$ & $E_3$ & $\{E_1,E_2,E_3\}$ \\
			\hline
			4 & $\{\family{1,\emptyset},\family{2,\{1\}}\}$ & $\{1,2\}$ & $G_2$ & $E_2$ & $\{E_1,E_2\}$ \\
		\end{tabular}
		\caption{A summary of the example trace}\label{tab:example_run}
		\bigskip
		
		\begin{tabular}{|c|c|c|}
			\hline
			$E_1$ & $E_2$ & $E_3$ \\
			\hline
			$1\rightarrow2\rightarrow3\rightarrow4$ & $1\rightarrow2\rightarrow4\rightarrow3$ &  
			$2\leftarrow1\rightarrow3\rightarrow4$\\
			$G_1$ & $G_2$ & $G_3$ \\
			\hline
			$1\leftarrow2\rightarrow3\rightarrow4$ & $1\leftarrow2\rightarrow4\rightarrow3$ &
			$2\rightarrow1\rightarrow3\rightarrow4$  \\
			
			$G_4$ & $G_5$ & $G_6$ \\
			\hline
			$1\leftarrow2\leftarrow3\rightarrow4$ & $1\leftarrow2\leftarrow4\rightarrow3$ & $2\leftarrow1\leftarrow3\rightarrow4$ \\
			
			$G_7$ & $G_8$ & $G_9$ \\
			\hline
			
			$1\leftarrow2\leftarrow3\leftarrow4$ & $1\leftarrow2\leftarrow4\leftarrow3$ & $2\leftarrow1\leftarrow3\leftarrow4$ \\
			
			$G_{10}$ & $G_{11}$ & $G_{12}$ \\
			\hline
			
		\end{tabular}
		\caption{Three ECs in $\ecset_{4,1}$. Each of these ECs includes four structures.}\label{tab:example_ecs}

	\end{table} 
	
	\paragraph{An execution example.} To demonstrate how \algname\ works, we now give an example trace for \algname\ with $d=4,k=1.$ A summary of this trace is provided in \tabref{example_run}. $\mathcal{F}_{4,1}$ includes four families with an empty parent set, and twelve families with a parent set of size one.
	\tabref{example_ecs} lists three of the ECs in $\ecset_{4,1}$ that we will refer to in this execution example. Note that  $T=\log_2(8)=3.$ 
	\begin{itemize}
		\item After the initialization stage (line \ref{init}), \algname\ draws $N_1$ samples from each pair of variables in $\{1,2,3,4\}$ (line \ref{sample}). 
		\item Suppose that the graph that maximizes the empirical score in line \ref{maximization} is $\hat{G}_1=G_3$, where $G_3$ is given in  \tabref{example_ecs}. Thus, in line \ref{Gj}, $\hat{E}_1=E_3.$ 
		\item Since $\accept_1$ is an empty set, $\struct_j = \struct_{4,1}$, and so $L_1$ (line \ref{L_j})  simply includes all the ECs that are empirically $\theta_1$-close  to $E_3.$ Suppose that $L_1=\{E_1,E_2,E_3\}.$
		\item The condition in line \ref{if} is satisfied by the families $\family{1,\emptyset}$ and $\family{2,\{1\}}$ (see $G_1,G_2,G_3$ in \tabref{example_ecs}), and by the family $\family{1, \{2\}}$ (see $G_4,G_5,G_6$ in \tabref{example_ecs}). 
		Suppose that $f:=\family{1,\emptyset}$ is selected in line \ref{condition} and so $f$ is accepted (line \ref{accept}). 
		Thus, $\accept_2=\{\family{1,\emptyset}\}$, and $\cV=\{1\}.$
		\item In the next iteration, $j=2$, the algorithm hasn't drawn any new samples. Therefore, since $G_3$ includes the accepted family, we have $\hat{G}_2=G_3,\hat{E}_2=E_3$ as in the previous iteration. Suppose that $L_2=L_1.$
		\item Note that only one graph in each EC in \tabref{example_ecs} includes the family $\family{1,\emptyset}$. Therefore, $E_1^{\cap 2}=\{G_1\},$ $E_2^{\cap 2}=\{G_2\},$ and $E_3^{\cap 2}=\{G_3\}$. The accepted family in this iteration can only be $\family{2,\{1\}}$, which is shared by $G_1,G_2$, and $G_3$.  Therefore, we now have $\accept_3=\{\family{1,\emptyset},\family{2,\{1\}}\},$ and $\cV=\{1,2\}.$
		\item  In iteration $j=3$, suppose that $L_3=L_2=\{E_1,E_2,E_3\}.$ $\hat{G}_3, \hat{E}_3$ and $E_i^{\cap 3}=\{G_i\}$ for $i\in\{1,2,3\}$ are the same as in the previous iterations. There is no family for $\{3,4\}$ that is shared by  $G_1,G_2$, and $G_3.$ Thus, in this iteration the condition on line \ref{if} does not hold, and no family is accepted on iteration 3. The condition ending the repeat-until loop  now holds, and \algname\ proceeds to the next round.
		\item On the next round, $t=2$, $N_2 - N_1$ additional samples are drawn, such that the total number of samples taken from each subset of size $2$ of variables (except for the subset $\{1,2\}$) is equal to $N_2$. Now  $j=4$.
		\item  Suppose that $\hat{G}_4=G_2$ in line \ref{maximization}. Thus, in line \ref{Gj}, $\hat{E}_4=E_2$. Suppose further that $L_4 = \{E_1,E_2\},$ and note that $E_2^{\cap 4}=\{G_1\}$ and $E_1^{\cap 4}=\{G_2\}.$ 
		Again, no family for $\{3,4\}$ is shared by $G_1$ and $G_2$, and no family is accepted at iteration 4. This causes the second round of the algorithm to terminate.
		\item The main loop now ends, since $T-1$ rounds have been completed. The final sampling round, round $T=3$, occurs outside of the main loop. Again all the subsets of size $2$ except for $\{1,2\}$ are sampled, so the total number of samples is $N_3$. \algname\ returns the highest-scoring graph that includes the accepted families $\family{1,\emptyset},\family{2,\{1\}}$.
	\end{itemize}

	\paragraph{The computational complexity of \algname.} As mentioned in \secref{related} above, score-based structure learning is NP-hard in the general case \citep{Chickering96}. Thus, \algname\ is not an efficient algorithm in the general case. Nonetheless, \algname\ is structured so that the computationally hard operations are encapsulated in two specific black boxes: 
	\begin{itemize}
		\item The score maximization procedure, used in lines \ref{maximization} and \ref{return}. A similar procedure is required by any score-based structure learning algorithm, including the naive algorithm proposed in \secref{naive}, as well as structure learning algorithms for the full information setting \cite[see, e.g.,][]{Chickering02, SilanderMy12,  BartlettCu13}. Thus, established and successful efficient heuristics exist for this procedure. In \algname, the maximization is constrained to include specific families. This does not add a significant computational burden to these heuristics. 
		\item The assignment of ECs that satisfy the required constraints in line \ref{L_j}. This procedure can be divided into two steps:
		\begin{enumerate}
			\item Compute all the high-scored structures that are consistent with the accepted families;
			\item Aggregate equivalent structures into ECs.
		\end{enumerate}
		
		Step 1 can be replaced with an efficient heuristic proposed in \cite{LiaoShCuVa18}, which is derived from score-maximization heuristics. 
		 
		Step 2 requires computing the EC of each of these structures, which can be done efficiently. 
		
	\end{itemize}
	This encapsulation thus allows using heuristics for these black boxes to obtain a practical algorithm, as we show in \secref{experiments} below, whereas an exact implementation of these black boxes would be exponential in $d$ in the worst case \cite[see, e.g.,][]{Chickering02, LiaoShCuVa18}.

	In the next section, we show that \algname\ indeed finds an $\epsilon$-optimal structure with a probability at least $1-\delta$.

	\section{Correctness of \algname}\label{sec:correct}
	The following theorem states the correctness of \algname.
	
	\begin{theorem}\label{thm:correct}
		With a probability at least $1-\delta$, the graph $\hat{G}$ returned by \algname\ satisfies $\cS(\hat{G}) \geq \cS^* - \epsilon.$
	\end{theorem}
	
	To prove \thmref{correct}, we first give some preliminaries. Let $\cV_t$ be the set $\cV$ at the beginning of round $t$. Define the following event:
	\[
	\event := \{\forall t \in [T],\forall f\in\actf(\cV_t), |\hat{H}_{t}(f)-H(f)|\leq\epsilon_{t}/2\}.
	\]
	Noting that for any family in $\actf(\cV_t)$, $\hat{H}_t(f)$ is estimated based on $N_t = N(\epsilon_t/2, \delta/(T|\cF_{d,k}|))$ samples, it is easy to see that by a union bound over at most $T$ rounds and at most $|\cF_{d,k}|$ families in each round, $\event$ holds with a probability at least $1-\delta$.

	Let an \emph{iteration} of \algname\  be a single pass of the inner loop starting at line \ref{iteration}. Let $J$ be the total number of iterations during the entire run of the algorithm. Note that a single round $t$ can include several iterations $j$. Denote by $\cV_{(j)}$ the value of $\cV$ at the start of iteration $j$. Under $\event$, it follows from above that at any iteration $j$ in round $t$, for any graph $G \in \struct_j$ (recall that we treat $G$ as the set of families that the graph consists of), 
	\begin{equation}\label{eq:thetat}
	|\score(G\setminus \accept_j) - \escore_t(G\setminus \accept_j)| \leq \sum_{f\in G\setminus \accept_j }|H(f) - \hat{H}_t(f)| \leq (d-|\cV_{(j)}|)\cdot\epsilon_t/2 = \theta_j/2,
	\end{equation}
	
	where $\theta_j$ is defined in line \ref{theta_j} of \myalgref{active_algorithm}. 
	We now give a bound on the empirical score difference between equivalence classes. For $i \in {1,2}$, let $E_i\in \ecset_{d,k}$ be two equivalence classes such that $\cons{E_i}{j} \neq \emptyset$, and $G_i \in \argmax_{G \in \cons{E_i}{j}} \escore_t(G)$. We have $G_i \in \struct_j$. Hence, $|\score(G_i\setminus \accept_j) - \escore_t(G_i\setminus \accept_j)| \leq \theta_j/2$.
	In addition, $\score(G_i\setminus \accept_j) = \score(E_i) - \score(\accept_j)$, and by the definition of the empirical score of a set of structures, we have that $\escore_t(\cons{E_i}{j}) = \escore_t(G_i) = \escore_t(G_i \setminus \accept_j) + \escore_t(\accept_j).$
	Therefore, under $\event$,
	\begin{align}\label{eq:thetatec}
	&\escore_t(\cons{E_1}{j}) - \escore_t(\cons{E_2}{j}) =
	\escore_t(G_1 \setminus \accept_j) - \escore_t(G_2 \setminus \accept_j) \notag\\
	&= (\escore_t(G_1 \setminus \accept_j) - \score(G_1 \setminus \accept_j))  +
	(\score(G_1 \setminus \accept_j) - \score(G_2 \setminus \accept_j)) \notag\\
	&\quad+(\score(G_2 \setminus \accept_j) - \escore_t(G_2 \setminus \accept_j))  \notag\\
	&\leq \score(G_1 \setminus \accept_j) - \score(G_2 \setminus \accept_j) + \theta_j\notag\\
	&= \score(E_1) - \score(E_2) + \theta_j.
	\end{align}
	
	The following lemma shows that if a family is accepted in the main loop of \myalgref{active_algorithm}, then it is in some optimal structure which includes all the families accepted so far.
	This shows that in \algname, there is always some optimal structure that is consistent with the families accepted so far. 
	
	Denote the set of optimal graphs in $\struct_j$ by $\struct_j^* := \struct^* \cap \struct_j$.
	
	\begin{lemma}\label{lem:family_correctness}
		Assume  that $\event$ occurs.  Then for all $j \in [J]$, $\struct_{j}^*\neq\emptyset$.
		
	\end{lemma}
	
	\begin{proof}
		The proof is by induction on the iteration $j$. For $j=1$, we have $\struct_j = \struct_{d,k}$, hence $\struct_1^* = \struct^* \neq \emptyset$. 
		Now, suppose that the claim holds for iteration $j$. Then $\struct_j^* \neq \emptyset$. 		We show that if any family $f$ is accepted in this iteration, it satisfies $f \in \uni{\struct^*_j}$. This suffices to prove the claim, since $\struct^*_j \subseteq \struct_j$  implies that all the graphs in $\struct^*_j$ include $\accept_{j}$; Combined with $f \in \uni{\struct^*_j}$, this implies that there is at least one graph in $\struct^*_j$ that includes $\accept_{j+1} \equiv \accept_{j}\cup \{f\}$, thus $\struct_{j+1}^* \neq \emptyset$, as needed.

		Suppose that $f$ is accepted at iteration $j$.
		Let $E_* \subseteq \struct^*$ be an optimal EC such that $\cons{E_*}{j}\neq\emptyset$; one exists due to the induction hypothesis.
		We show that $f \in \uni{\cons{E_*}{j}}$, thus proving that $f \in \uni{\struct^*_j}$. By event $\event$ and \eqref{thetatec}, we have
		$\escore_t(\cons{\hat{E}_j}{j}) - \escore_t(\cons{E_*}{j}) \leq \score(\hat{E}_j) - \score(E_*) + \theta_j \leq \theta_j.$
		
		The last inequality follows since $\score(\hat{E}_t) \leq \score(E_*)$.
		Therefore, for $L_j$ as defined in line~\ref{L_j} of \myalgref{active_algorithm}, we have $E_* \in L_j$. Since $f$ is accepted, it satisfies the condition in line \ref{if}, therefore $f \in \uni{\cons{E_*}{j}}$. This proves the claim.                
		         
	\end{proof}

	Using \lemref{family_correctness}, the correctness of \algname\ can be shown.
	\begin{proof}[Proof of \thmref{correct}]
		Assume that $\event$ holds. As shown above, this occurs with a probability at least $1-\delta$. Let $\accept := \accept_J$ be the set of accepted families at the end of the main loop of \myalgref{active_algorithm}. For any $G \in \struct_J$, we have $\accept \subseteq G$, hence $\score(G) = \score(\accept) + \score(G \setminus \accept)$, and similarly for $\escore_T$. 
		
		In addition, for each family $f \in G \setminus \accept$, we have  $f \in \actf(\cV_T)$, since $\actf(\cV_T)$ includes all the families of a child node whose family was not yet accepted. So, by event $\event$ we have $|\hat{H}_T(f) - H(f)| \leq \epsilon_{\mathrm{last}}/2$, where $\epsilon_{\mathrm{last}}$ is defined in line~\ref{step:epslast} of \myalgref{active_algorithm}.
		Therefore, as in \eqref{thetat}, for any $G \in \struct_J$, 
		\[
		|\score(G \setminus \accept) -\escore_T(G \setminus \accept)|
		\leq (d - |\cV_T|)\cdot \epsilon_{\mathrm{last}}/2 = \epsilon/2.
		\]
		By \lemref{family_correctness}, we have that $\struct_J^*\neq\emptyset$. Let $G^*$ be some graph in $\struct_J^* \subseteq \struct_J$. Recall that also $\hat{G} \in \struct_J$.
		By the definition of $\hat{G}$, $\escore_T(\hat{G}) \geq \escore_T(G^*)$.
		Hence, $\escore_T(\hat{G} \setminus \accept) \geq \escore_T(G^* \setminus \accept)$.
		Combining this with the inequality above, it follows that
		\[
		\score(\hat{G} \setminus \accept) \geq \escore_T(\hat{G} \setminus \accept) - \epsilon/2 \geq \escore_T(G^* \setminus \accept) - \epsilon/2 \geq \score(G^* \setminus \accept) - \epsilon.
		\]
		Therefore, $\score(\hat{G}) \geq \score(G^*) - \epsilon$, which proves the claim.
	\end{proof}

	\section{The Sample Complexity of \algname}\label{sec:sample}
	
	We now analyze the number of samples drawn by \algname. We show that it is never much larger than that of the naive algorithm, while it can be significantly smaller for some classes of distributions. For concreteness, we show sample complexity calculations using the value of $N(\epsilon,\delta)$ given in \secref{naive} for discrete distributions with the plug-in entropy estimator. We assume a  regime with a small constant $k$ and a possibly large $d$. 
	
	\algname\ requires the largest number of samples if no family is accepted until round $T\equiv\lceil \log_2( 2d) \rceil$. In this case, the algorithm pulls all subsets in $\subsets{d}$, each for $N(\epsilon/(2d),\delta/(T|\cF_{d,k}|))$ times. Thus, the worst-case sample complexity of \algname\ is
	\[
	|\subsets{d}| \cdot N(\epsilon/(2d),\delta/(T|\cF_{d,k}|)) = \widetilde{\Theta}\left(\frac{kd^{k+3}}{\epsilon^2}\log( 1/\delta) \right).
	\]

	In the general case, the sample complexity of \algname\ can be upper-bounded based on the round in which each variable was added to $\cV$. Let $t_1 < \ldots < t_n$ be the rounds in which at least one variable was added to $\cV$, and denote $t_{n+1} := T$. Let $V_i$ for $i \in [n]$ be the total number of variables added until the end of round $t_i$, and set $V_0 = 0, V_{n+1} = d$. Note that $V_{n}$ is the size of $\cV$ at the end of \algname.
	In round $t_i$, $V_{i} - V_{i-1}$ variables are added to $\cV$. Denoting by $Q_i$  the number of subsets that are observed by \algname\ for the last time in round $t_i$, we have
	\[
	Q_{i} := |\subsets{V_{i}} \setminus \subsets{V_{i-1}}| = \binom{V_{i}}{k+1} - \binom{V_{i-1}}{k+1} \leq (V_i - V_{i-1})V_i^k.
	\]
	The sample complexity of \algname\ is at most $\sum_{i \in [n+1]} Q_{i} N_{t_i}$, where
	\[
	N_{t_i} \equiv N(\epsilon_{t_i}/2,\delta/(T|\cF_{d,k}|)) = \widetilde{O}\left(\frac{1}{\epsilon_{t_i}^2}k\log(d/\delta)\right).
	\]

	The sample complexity is thus upper-bounded by
	\begin{equation*}\label{eq:tsc}
	 \widetilde{O}\left( \Big(\sum_{i \in [n]}  (V_i - V_{i-1})V_i^k\cdot\frac{1}{\epsilon^2_{t_i}} + kd^k(d-V_n)^3 \cdot \frac{1}{\epsilon^2}\Big) \log(d/\delta)\right).
	\end{equation*}
	
	Since $\epsilon_t$ is decreasing in $t$, if most
	variables are added to $\cV$ early in the run of the algorithm, considerable savings in the sample complexity are possible.  
	We next describe a family of distributions for which \algname\ can have a significantly smaller sample complexity compared to the naive algorithm. 
	In the next section, we describe a general construction that satisfies these requirements.

	\begin{definition}\label{def:gammastable}
		Let $\gamma > 0$, and let $V \subseteq \mathbf{X}$. A \emph{$(\gamma,V)$-stable distribution} over $\mathbf{X}$ is a distribution which satisfies the following conditions.
		\begin{enumerate}
			\item In every $G \in \struct_{d,k}$ such that $\score(G) \geq \score^* - \gamma$, the parents of all the variables in $V$ are also in $V$. \label{vgap}
			\item There is a unique optimal EC for the marginal distribution on $V$, and the difference in scores between the best EC and the second-best EC on $V$ is more than $\gamma$.  \label{unique}
		\end{enumerate}
		
	\end{definition}
	
	\thmref{stableimprove}, stated below, states that the sample complexity improvement for a discrete stable distribution can be as large as a factor of $\widetilde{\Omega}(d^3)$. In \secref{stableex}, we give a specific construction of a class of distributions that satisfies the necessary conditions of \thmref{stableimprove}. This class is characterized by the existence of two similar variables, one a noisy version of the other.
	
	\begin{theorem}\label{thm:stableimprove}
		 Let $\gamma > 0$. Let $v := |V|$. 
		Let $\cD$ be a discrete $(\gamma,V)$-stable distribution, and assume that \algname\ samples from $\cD.$
		Then, the sample complexity of  \algname\ is at most 
		\begin{equation}\label{eq:stable_sc}
		\widetilde{O}\left(\left(\frac{d^2v^{k+1}}{\gamma^2}+\frac{d^k(d-v)^3}{\epsilon^2} \right)\cdot k \log(1/\delta)\right).
		\end{equation}
		Furthermore, if $v=d-O(1)$ and $\epsilon=O(\gamma d^{-3/2})$, then the sample complexity improvement factor of \algname\ compared to the naive algorithm is $\widetilde{\Omega}(d^3).$  
	\end{theorem}

	To prove \thmref{stableimprove}, we first prove several lemmas.
	The following lemma gives a sufficient condition for a family to be accepted by \algname\ at a given iteration. 
	\begin{lemma}\label{lem:2nd_direction}
		Assume that $\event$ occurs. 
		Let $j$ be an iteration in the run of \algname. For a family $f$, define
		\[
		S_j^\neg(f) := \max\{\score(E) \mid E \in \ecset_{d,k}, \cons{E}{j} \neq \emptyset,  f \notin \uni{\cons{E}{j}}\}.
		\]
		This is the maximal score of an equivalence class which is consistent with families accepted so far and also inconsistent with $f$.
		
		Suppose that for some $f \in \uni{\struct_j^*}\cap \actf(\cV_{(j)})$, we have $S_j^\neg(f) < \score^* - 2\theta_j$. Then, some family is accepted by \algname\ at iteration $j$.
		 
	\end{lemma}

	\begin{proof}
		Let $E_* \subseteq \struct^*$ be an optimal EC such that $\cons{E_*}{j}\neq\emptyset$. By \lemref{family_correctness}, such an EC exists. 
		Suppose that the assumption of the lemma holds. Then, for some $f\in \uni{\struct_j^*}\cap \actf(\cV_{(j)})$, we have $S_j^\neg(f) < \score^* - 2\theta_j$. 
		It follows that for any $E \in \ecset_{d,k}$ such that $\cons{E}{j} \neq \emptyset$ and $f \notin \uni{\cons{E}{j}}$, $2\theta_j<\score(E_*)-\score(E)$. By the definition of $\hat{E}_j$ (line \ref{Gj} in \myalgref{active_algorithm}), we have  $\escore_t(\cons{\hat{E}_j}{j}) \geq \escore_t(\cons{E_*}{j})$. In addition, since event $\event$ occurs, \eqref{thetatec} holds.  
		Therefore, 
		
		\[
		\escore_t(\cons{\hat{E}_j}{j})-\escore_t(\cons{E}{j}) \geq \escore_t(\cons{E_*}{j})-\escore_t(\cons{E}{j}) \geq \score(E_*) -\score(E) - \theta_j > \theta_j.
		\]

		It follows that $E \notin L_j$, where $L_j$ is as defined in line~\ref{L_j} of \myalgref{active_algorithm}. Therefore, $\forall E \in L_j, f \in \uni{\cons{E}{j}}$. In particular, $f \in \uni{\cons{\hat{E}_j}{j}}$. Since  $f \in \actf(\cV_{(j)})$, it follows that $f \in \uni{\cons{\hat{E}_j}{j}} \cap \actf(\cV_{(j)})$. Therefore, the condition in line \ref{if} is satisfied, which implies that some family is accepted during iteration $j$. 
	\end{proof}
	
	Next, we give a characterization of the optimal structures for $\cD$. Let $\struct_V$ be the set of DAGs with an in-degree at most $k$ over $V$. Denote $\bar{V} := \mathbf{X} \setminus V$. We term a set of families $F \subseteq \cF_{d,k}$ a \emph{legal family set for $\bar{V}$} if it includes exactly one family for each variable in $\bar{V}$ and no other families, and has no cycles. First, we provide an auxiliary lemma.
	
	\begin{lemma}\label{lem:legaldag}
		If $G\in\struct_V$ and $F$ is a legal family set for $\bar{V}$, then $G \cup F \in \struct_{d,k}$.
	\end{lemma}
	
	\begin{proof}
		Denote $G':=G\cup F.$
		First, we show that $G'$ has an indegree of at most $k.$
		Any $G\in\mathcal{G}_V$ has exactly one family for every variable in $V$ with at most $k$ parents. Thus, in $G'$, there is exactly one family for each variable in $V\cup\bar{V}$, and each family includes at most $k$ parents. Thus, the degree constraint is not violated in $G'$.
		
		Next, we address the acyclicity constraint. 
		Assume in contradiction that $G'$ contains a cycle, and let $(A,B)$ be an edge in the cycle. This means that there is a directed path from $B$ to $A$ in $G'$. Let $(B,v_1,\ldots,v_m,A)$ be the nodes in the path. Note that $G$ includes only edges from $V$ to $V$, and $F$ includes edges from $V\cup \bar{V}$ to $\bar{V}.$  Consider the possible cases:
		\begin{enumerate}
			\item $A\in \bar{V}$, $B\in V$: Neither $G$ nor $F$ include edges from $\bar{V}$ to $V$. Thus, this case is impossible.
			\item $A\in V,$ $B\in V$: In this case, $(A,B)$ must be an edge in $G.$ In addition, since there are no edges from $\bar{V}$ to $V$ in $G'$, all the nodes $v_1,\ldots,v_m$ on the path from $B$ to $A$ are also in $V.$ This means that this path is entirely in $G$, which contradicts the acyclicity of $G.$
			\item $A\in V,$ $B\in \bar{V}$: In this case, the path from $B$ to $A$ includes at least one edge from $\bar{V}$ to $V$. However, neither $G$ nor $F$ include edges from $\bar{V}$ to $V$. Therefore, this case is impossible.
			\item $A\in \bar{V},$ $B\in \bar{V}$: In this case, $(A,B)$ must be in $F.$ In addition, since there are no edges from $\bar{V}$ to $V$ in $G'$, all the nodes $v_1,\ldots,v_m$ on the path form $B$ to $A$ are also in $\bar{V}.$ This means that this path is entirely in $F$, which contradicts the acyclicity of $F.$
			
		\end{enumerate}
		
		Since all cases are impossible, the existence of a cycle in $G'$ is impossible.
		This completes the proof.
	\end{proof}
	
	Next, we provide the following characterization of the optimal structures for $\cD$.
	 
	\begin{lemma}\label{lem:optimalstruct}
		Let $\cD$ be a $(\gamma,V)$-stable distribution.  Let $G^* \in E_*$, and let $F \subseteq G^*$ be set of the families of $\bar{V}$ in $G^*$.
		Then the following hold:
		\begin{enumerate}
			\item $G^* \setminus F \in \struct_V$;
			\item $\score(G^* \setminus F) = \max_{G \in \struct_V} \score(G)$;
			\item The score of $F$  is maximal among the legal family sets for $\bar{V}$.
		\end{enumerate}

	\end{lemma}
	\begin{proof}
		By the assumptions on the distribution, since $G^*$ is an optimal structure, $\bar{V}$ has no outgoing edges into $V$. Therefore, $G^* \setminus F \in \struct_V$. This proves the first part of the lemma.

		To prove the second part of the lemma, let $G^*_1 \in \struct_V$ be some DAG over $V$ with a maximal score.  Note that $F$ is a legal family set for $\bar{V}$. 
		Let $G_1 := G^*_1 \cup F.$ By \lemref{legaldag}, $G_1\in \struct_{d,k}$. Then,
		\[
		\score(G^*\setminus F) + \score(F) = \score(G^*) \geq \score(G_1) = \score(G_1^*) + \score(F).
		\]
		Therefore, $\score(G^*\setminus F) \geq \score(G^*_1)$. Hence, $G^* \setminus F$ is also optimal over $V$. This proves the second part of the lemma. 
		
		To prove the third part of the lemma, let $F'\neq F$ be some legal family set for $\bar{V}$. Then, by \lemref{legaldag}, $G_2 := (G^* \setminus  F) \cup F' \in \struct_{d,k}$, 
		and $\score(G_2) = \score(G^*) - \score(F) + \score(F')$. 
		Since $G^*$ is an optimal structure, $\score(G_2) \leq \score(G^*)$. Therefore, $\score(F) \geq \score(F')$. This holds for all legal family sets of $\bar{V}$. Therefore, the score of $F$ is maximal among such sets, which proves the third part of the lemma.
	\end{proof}

	The next lemma shows that equivalence classes with a near-optimal score include graphs with common families.

	\begin{lemma}\label{lem:existetilde}
		Let $\cD$ be a $(\gamma,V)$-stable distribution. For a set of families $A\in\cF_{d,k}$ and an equivalence class $E$, denote $E^{\cap A}:=\{G \in E \mid A\subseteq G\}$.   
		Let $G^* \in E_*^{\cap A}$, and suppose that there is some $\tilde{E} \in \ecset_{d,k}$ such that $\tilde{E} \neq E_*$, $\score(\tilde{E}) \geq \score^* - \gamma$ and $\tilde{E}^{\cap A} \neq \emptyset$. 
		
		Then, there exists a graph $\tilde{G} \in \tilde{E}^{\cap A}$ such that the families of all the variables in $V$ are the same in both $G^*$ and $\tilde{G}$. 
	\end{lemma}

	To prove this lemma, we use a known characterization of a Markov equivalence class, which relies on the notion of a \emph{v-structure}.\footnote{Also known as an "immorality" \cite[see, e.g.,][]{VermaPe90}} In a DAG $G$, a v-structure is an ordered triplet of nodes $(X,Y,Z)$ such that $G$ contains the edges $(X,Y)$ and $(Z,Y)$, and there is no edge in either direction between $X$ and $Z$ in $G$. 
	
	\begin{theorem}[\citealt{VermaPe91}]\label{thm:vstructures}
		Two DAGs are equivalent if and only if they have the same skeletons (underlying undirected graphs) and the same set of v-structures.
		
	\end{theorem}
	
	\begin{proof}[of \lemref{existetilde}]
		Let $G$ be some graph in $\tilde{E}^{\cap A}$.  
		Let $F$, $F^*$ be the families of $\bar{V}$ in
		$G$ and $G^*$ respectively. Note that both $F$ and $F^*$ are legal family sets for $\bar{V}$, since they are subsets of DAGs in $\struct_{d,k}$. Define the graph $G_1 := G^* \setminus F^*$, and set $\tilde{G}:=G_1\cup F$. By \lemref{optimalstruct}, $G_1\in\struct_V$.
		Thus, by \lemref{legaldag}, we also have  $\tilde{G}\in\struct_{d,k}$.  
		Furthermore, all of the variables in $V$ have the same families in both $G^*$ and $\tilde{G}.$
		
		To show that $\tilde{G}$ satisfies the conditions of the lemma, it is left to show that $\tilde{G} \in \tilde{E}^{\cap A}$. First, note that $A \subseteq \tilde{G}$, since any family in $A$ is either in $G_1$ or in $F$. 
		We now prove that $\tilde{G}\in\tilde{E}$, by showing that $G$ and $\tilde{G}$ are Markov equivalent. Define $G_2:= G \setminus F$. Since $\score(G)=\score(\tilde{E})\geq\score^*-\gamma$, by assumption \ref{vgap} in \defref{gammastable}, it holds that in $G$ all the variables in $V$ have parents in $V$. Thus, $G_2 \in \struct_v$.

		We now prove that $G_1$ and $G_2$ are Markov equivalent, and conclude that the same holds for $G$ and $\tilde{G}$. First, we show that $\score(G_2) = \score(G_1)$. 
		We have $G_1,G_2 \in \struct_v$, and by \lemref{optimalstruct}, $G_1$ is an optimal graph on $V$. Therefore, $\score(G_2)\leq \score(G_1)$. 
		Now, suppose for contradiction that $\score(G_2)<\score(G_1)$.
		Then, by assumption \ref{unique} in \defref{gammastable}, we have $\score(G_2) < \score(G_1) - \gamma$. We also have, by \lemref{optimalstruct}, that $\score(F) \leq \score(F^*)$. It follows that
		\[
		\score(G)=\score(G_2)+\score(F)\leq\score(G_2)+\score(F^*) < \score(G_1) - \gamma +\score(F^*)=\score(G^*) - \gamma = \score^* -\gamma.
		\]
		
		Therefore, $\score(G) < \score^* -\gamma$. 
		But $G \in \tilde{E}$, and so $\score(G) \geq \score^* -\gamma$, leading to a contradiction. It follows that $\score(G_2) = \score(G_1)$, implying that $G_2$ is also an optimal structure on $V$. 
		Combining this with the uniqueness of the optimal EC on $V$, as given in assumption \ref{unique} of \defref{gammastable}, we conclude that $G_1$ and $G_2$ are Markov equivalent.
		
		To show that $G$ and $\tilde{G}$ are Markov equivalent, we first observe that since $G_1$ and $G_2$ are equivalent, then by \thmref{vstructures}, they have the same skeleton and the same set of v-structures.
		Therefore, $G_1 \cup F$ and $G_2 \cup F$ also have the same skeleton.
		In addition, they have the same set of v-structures, as follows: The v-structures with a child in $V$ are in $G_1$ and $G_2$ and so they are shared; The v-structures with a child in $\bar{V}$ are those with parents in $F$ and no edges between the parents. Since both graphs share the same skeleton, these must be the same v-structures. This proves that $G$ and $\tilde{G}$ are Markov equivalent, thus $\tilde{G} \in \tilde{E}$, and so, as observed above, also $\tilde{G} \in \tilde{E}^{\cap A}$, which completes the proof.
		
	\end{proof}

	The last lemma required for proving \thmref{stableimprove} shows that for a stable distribution, \algname\ is guaranteed to accept families early.
	\begin{lemma}\label{lem:outgoing}
		Let $\cD$ be a $(\gamma,V)$-stable distribution. Consider a run of \algname\ in which $\event$ holds.
		Then, \algname\ accepts at least $|V|$ families by the end of the first round $t$ which satisfies $\epsilon_t \leq \gamma/(2d)$.  
	\end{lemma}
	
	\begin{proof}
		
		Suppose that \algname\ has accepted less than $|V|$ families until some iteration $j$. Then $V \setminus \cV_{(j)} \neq \emptyset$. Let $f \in \uni{\struct_j^*} \cap \actf(\cV_{(j)})$ be a family of some variable in $V\setminus \cV_{(j)}$ that belongs to an optimal structure in $\struct^*_j$. Note that such a family exists, since by \lemref{family_correctness}, $\struct^*_j \neq \emptyset$. 
		Let $\tilde{E} \in \ecset_{d,k}$ such that $\cons{\tilde{E}}{j} \neq \emptyset$ and $\score(\tilde{E}) \geq \score^* - \gamma$. 
		Let $G^* \in \struct^*_j$. By \lemref{existetilde} with $A := \accept_j$,  there exists a graph $\tilde{G} \in \cons{\tilde{E}}{j}$ such that $f \in \tilde{G}$. Therefore, $f \in \uni{\cons{\tilde{E}}{j}}$. 
		Since this holds for any $\tilde{E} \in \ecset_{d,k}$ that satisfies the conditions above, it follows that for $S_j^\neg(f)$ as defined in \lemref{2nd_direction}, 
		we have that $S_j^\neg(f) < \score^* - \gamma$. The conditions of \lemref{2nd_direction} thus hold if $2\theta_j \leq \gamma$. In this case, some family will be accepted at iteration $j$.
		 
		In round $t$, $\theta_j \leq d\epsilon_t$, and the round only ends when no additional families are accepted. Therefore, at least $|V|$ families will be accepted until the end of the first round $t$ with $\epsilon_t \leq \gamma/(2d)$. 
	\end{proof}

		We are now ready to prove \thmref{stableimprove}.
	
	\begin{proof}[of \thmref{stableimprove}]
		Since $\cD$ is a discrete distribution, the sample complexity upper bounds 
		in \eqref{tsc} and  \eqref{naive} can be used. 
		To upper bound the sample complexity of \algname, observe that by \lemref{outgoing}, on the first round $t$ with $\epsilon_t \leq \gamma/(2d)$, at least $|V| = v$ families have been accepted. 
		Following \eqref{tsc} and the notation therein, the sample complexity of \algname\ can be upper bounded by setting $n=1$, $V_1 = v$, $t_1 = t$, and getting an upper bound of 
		\begin{align*}
		&\widetilde{O}\left( \Big(v^{k+1}\cdot\frac{1}{\epsilon^2_t} +  d^k(d-v)^3 \cdot \frac{1}{\epsilon^2}\Big) \cdot \log(d^{k+1}/\delta)\right)
		&= \widetilde{O}\left(\Big(\frac{d^2v^{k+1}}{\gamma^2} + \frac{d^k(d-v)^3}{\epsilon^2}\Big)\cdot k\log(1/\delta)\right),
		\end{align*}

		where we used the fact that $\epsilon_t \in (\gamma/(4d),\gamma/(2d)]$. This proves \eqref{stable_sc}. 
		
		By \eqref{naive}, the sample complexity of the naive algorithm is 
		$\tilde{\Omega}( \frac{d^{k+3}}{\epsilon^2}\cdot k\log(1/\delta)).$ 
		Substituting $v = d-O(1)$ in \eqref{stable_sc}, we get that \algname\ uses $\widetilde{O}\left((\frac{d^{k+3}}{\gamma^2}+\frac{d^k}{\epsilon^2})\cdot k\log(1/\delta)\right)$ samples. Dividing this by the sample complexity of the naive algorithm, we get a sample complexity ratio of $\widetilde{O}((\epsilon/\gamma)^2+ 1/d^3)$. For $\epsilon = O(\gamma d^{-3/2})$, this ratio is $\widetilde{O}(1/d^3)$. The sample complexity of \algname\ is thus a factor of $\widetilde{\Omega}(d^3)$ smaller than that of the naive algorithm in this case.
		This completes the proof. 
	\end{proof}
	
	\thmref{stableimprove} gives a set of conditions that, if satisfied by a distribution, lead to a significant sample complexity reduction when using \algname. In the next section, we give a construction which satisfies this set of conditions.

	\section{A Class of Stable Distributions} \label{sec:stableex}
	
	In this section, we give an explicit construction of a family of distributions that are stable according to \defref{gammastable}. The construction starts with a distribution with a unique optimal EC, and augments it with a noisy version of one of its variables. This leads to a structure learning problem in which most variable dependencies are easy to identify, but it is difficult to discern between the two versions of the noisy variable. In this situation, the advantage of \algname\ is manifested, since it can request more samples from the families that are more difficult to choose from.
	
	Let $\cD_1$ be a distribution over a finite set of at least $k$ variables $\mathbf{X}_1$, one of which is denoted $X_a$, which satisfies the following properties for some values $\beta,\alpha > 0$.
	\begin{enumerate}[label=(\Roman*)]
		\item $\cD_1$ has a unique optimal EC, and the difference in scores between this EC and the next-best EC is at least $\beta$. \label{gap} 
		\item $\xa\in\mathbf{X}_1$ does not have children in any of the structures in the optimal EC \label{xdchildless}
		\item $H(\xa \mid \mathbf{X}_1 \setminus \{\xa\}) = \alpha$. \label{alpha}
	\end{enumerate}
	Property \ref{xdchildless} holds, for instance, in the case that the optimal EC includes a graph in which $\xa$ has no children and no edges between any of its parents (see \figref{stable_illustration} for illustration). By \thmref{vstructures}, in this case, $\xa$ is the child in a v-structure with each two of its parents, and since v-structures are preserved within the optimal EC, $\xa$ has no children in any of the equivalent structures. 
	
	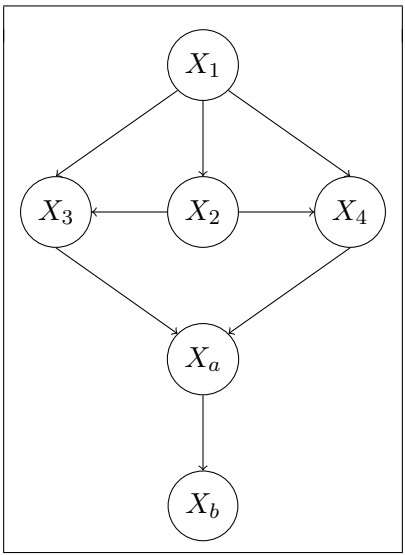
\begin{figure}[h!]
		\begin{center}
			\begin{tabular}[ht]{| c |}
				\hline
				\\[-5pt]
				\begin{tikzpicture}[
				observednode/.style={circle, draw},
				hiddennode/.style={circle, draw, dashed},
				]
				\node[observednode] (x0) {$X_1$};
				\node[observednode, below=of x0] (x1) {$X_2$};
				\node[observednode, left=of x1] (x2) {$X_3$};
				\node[observednode, right=of x1] (x3) {$X_4$};
				\node[observednode, below=of x1] (x4) {$\xa$};
				\node[observednode, below=of x4] (x5) {$\xb$};

				\draw[->] (x0.south) -- (x1.north);
				\draw[->] (x0.south west) -- (x2.north);
				\draw[->] (x0.south east) -- (x3.north);
				\draw[->] (x1.west) -- (x2.east);
				\draw[->] (x1.east) -- (x3.west);
				\draw[->] (x2.south) -- (x4.north west);
				\draw[->] (x3.south) -- (x4.north east);
				\draw[->] (x4.south) -- (x5.north);

				\end{tikzpicture} \\
				
				\hline
			\end{tabular}
			
			\caption{An illustration of an EC that satisfies property \ref{xdchildless}. }\label{fig:stable_illustration}
		\end{center}
	\end{figure}
	
	We now define a distribution which is similar to $\cD_1$, except that $\xa$ has a slightly noisy copy, denoted $\xb$. Let the set of variables be $\mathbf{X} := \mathbf{X}_1 \cup \{ \xb\} $, and set $d := |\mathbf{X}|$. For any $\lambda \in\left(0,\min(\alpha,\frac{\beta}{3d})\right)$, we define the distribution $\cD_2(\lambda)$ of $\mathbf{X}$ in which the marginal of $\cD_2(\lambda)$ over $\mathbf{X}_1$ is $\cD_1$, and $\xb$ is defined as follows:
	\begin{enumerate}[label=(\Roman*)]\setcounter{enumi}{3}
		\item $\xb$ is an independently noisy version of $\xa$. Formally, there is a hidden Bernoulli variable $C$ which is independent of $\mathbf{X}_1$, such that $P(\xb = \xa \mid C = 1) = 1$, and $(\xb \bot \mathbf{X}_1) \mid C=0$.\label{coin}
		
		\item The probability $P(C = 1)$ is such that $\max(H(\xb \mid \xa) ,H(\xa\mid \xb))=\lambda$. \label{lambda}
	\end{enumerate}

	The following theorem shows that $\cD_2$ is stable where the set $V$ includes all but one variable. Thus, \thmref{stableimprove} holds with $q = 1$, leading to a sample complexity improvement factor as large as $\widetilde{\Omega}(d^3)$ for this construction. 
	
	\begin{theorem}\label{thm:d2main}
		$\cD_2(\lambda)$ is a $(\gamma, V)$-stable distribution for $V := \mathbf{X}_1 \setminus \xa$ and for all $\gamma \leq \beta-3d\lambda$. 
	\end{theorem}
	Before turning to the proof of \thmref{d2main}, we note that \thmref{stableimprove} assumes a given required accuracy $\epsilon$. If the score of the second-best EC is much smaller than $\score^* - \epsilon$, and this is known in advance, then one could set $\epsilon$ to a larger value, thus potentially voiding the conclusion of \thmref{stableimprove}. The following theorem shows that this is not the case for the construction of $\cD_2$. Its proof is provided in \appref{deferred}. 
	\begin{theorem}\label{thm:scoreeps}
		For  $\epsilon \in \left(0,2\cdot\min(\alpha,\frac{\beta}{3d})\right)$, there exists a value of $\lambda$ such that in $\cD_2(\lambda)$ there exists a non-optimal EC with a score of at least $S^* - \epsilon$.
	\end{theorem}

	We now prove \thmref{d2main} by proving several lemmas. The first lemma gives useful technical inequalities. Its proof is provided in \appref{deferred}.
	\begin{lemma}\label{lem:help}
		Assume the distribution $\cD_2(\lambda)$ defined above. 
		For any $\Pi,\Pb,\Pa \subseteq \mathbf{X}_1$ and any $Y \in \mathbf{X}_1$, the following holds:
		\begin{align}
		&|H(\xb\mid\Pi)-H(\xa\mid\Pi)|\leq\lambda,\label{eq:xmidpi}\\
		&|H(Y\mid\Pi,\xb)-H(Y\mid\Pi,\xa)|\leq3\lambda.\label{eq:y}
		\end{align}

	\end{lemma}

	The next lemma shows that any DAG for $\cD_2$ can be transformed to a DAG with a similar score in which $\xb$ has no children. The proof of the lemma is provided in \appref{deferred}.
	\begin{lemma}\label{lem:no_children} 
		Let $G \in \struct_{d,k}$. There exists a graph $\tilde{G}\in \struct_{d,k}$ that satisfies the following properties: \begin{enumerate}
			\item $|\score(\tilde{G})-\score(G)|<3d\lambda$;\label{sim_score}
			\item In $\tilde{G}$,  $\xb$ has no children; \label{childless}
			\item The children of $\xa$ in $\tilde{G}$ are exactly the union of the children of $\xa$ and the children of $\xb$ in $G$ (except for $\xa$, if it is a child of $\xb$ in $G$)\label{xa_child}.
		\end{enumerate}
	\end{lemma}

	Next, we use \lemref{no_children} to prove that property \ref{vgap} in \defref{gammastable} holds for $\cD_2$. 
	\begin{lemma}\label{lem:optimalecgap}
		Assume the distribution $\cD_2(\lambda)$ as defined above. Let $G$ be a structure such that at least one of $\xa,\xb$ has children other than $\xa,\xb$. 
		Then, $\score(G) <\score^* - \beta + 3d\lambda$. 
	\end{lemma}
	\begin{proof}
		In this proof, all scores are calculated for the joint distribution defined by $\cD_2(\lambda)$. 
		Let $\Pb$ be a parent set of size $k$ for $\xb$  that maximizes the family score. Formally,
		\[
		\Pb \in \argmax_{\Pi\subseteq \mathbf{X}_1:|\Pi|=k}\score(\{\family{\xb,\Pi}\}).
		\]
		Let $f^*_b = \family{\xb, \Pb}$. 
		Denote by $\struct_{d-1,k}$ the set of DAGs with an in-degree of at most $k$ over $\mathbf{X}_1$. Let $G_1 \in \struct_{d-1,k}$ be such a DAG with a maximal score, and denote this score $\bar{\score}$. Let $G_1' = G_1 \cup \{f^*_b\}$. This is a DAG in $\struct_{d,k}$ (over $\mathbf{X} \equiv \mathbf{X}_1\cup \{ \xb\}$), with a score  $\score(G_1') = \bar{\score} + \score(\{f^*_b\})$. Therefore, $\score^* \geq \bar{\score} + \score(\{f^*_b\})$.

		Now, let $G$ be a structure such that at least one of $\xa,\xb$ has children other than $\xa,\xb$. We consider two cases:
		\begin{enumerate}
			\item $\xb$ has no children in $G$.
			\item  $\xb$ has children in $G$.
		\end{enumerate}
		In case 1, 
		$\xa$ has children in $\mathbf{X}_1$. Let $f_b$ be the family of $\xb$ in $G$. The graph $G \setminus \{f_b\}$ is a DAG over $\mathbf{X}_1$ in which $\xa$ has children. Therefore, by assumption \ref{xdchildless} in the definition of $\cD_1$, $G\setminus \{f_b\}$ is not an optimal DAG, and so by assumption \ref{gap}, it holds that $\score(G\setminus \{f_b\}) \leq \bar{\score} - \beta$.
		 
		Therefore, we have 
		\[
		\score(G) = \score(G\setminus \{f_b\}) + \score(\{f_b\}) \leq \score(G\setminus \{f_b\}) + \score(\{f^*_b\}) \leq \bar{\score} - \beta + \score(\{f^*_b\}) \leq \score^* - \beta.
		\]
		The last inequality follows from the fact proved above that $\score^* \geq \bar{\score} + \score(\{f^*_b\})$. This proves the statement of the lemma for case 1.
		
		In case 2, $\xb$ has children in $G$. Denote this set of children by $W$. By \lemref{no_children}, there exists a graph $\tilde{G} \in \struct_{d,k}$ such that $|\score(G)-\score(\tilde{G})| < 3d\lambda$, $\xb$ has no children, and the children of $\xa$ are $W \setminus \{\xa\}$ as well as the children of $\xa$ in $G$. Now, consider two cases.
		\begin{itemize}
			\item $W \neq \{\xa\}$. In this case, $W \setminus \{\xa\} \neq \emptyset$ but cannot include $\xb$, and so in $\tilde{G}$, $\xa$ has at least one child which is different from $\xb$.
			\item $W = \{\xa\}$. In this case, from the definition of $G$, it follows that $\xa$ has children other than $\xb$ in $G$, and so also in $\tilde{G}$. 
		\end{itemize}
		In both cases, we get that in $\tilde{G}$, $\xb$ has no children and $\xa$ has at least one child other than $\xb$. Thus, $\tilde{G}$ satisfies the conditions of case 1 discussed above. It follows that $\score(\tilde{G}) \leq \score^* - \beta$. Therefore, $\score(G) < \score^* - \beta + 3d\lambda$. This proves the statement of the lemma for case 2, thus completing the proof. 
	\end{proof}

	The next lemma shows that property \ref{unique} in \defref{gammastable} of a stable distribution holds for $\cD_2$, for any $\gamma < \beta$  
	\begin{lemma}\label{lem:gapd1}
		Consider the marginal of $\cD_2(\lambda)$ on $\mathbf{X}_1 \setminus \{\xa\}$. There is a unique optimal EC for this distribution, and the difference in scores between the optimal EC and the second-best EC is at least $\beta$. 
	\end{lemma}
	
	\begin{proof}
		Let $E_1$ be an optimal EC for the considered marginal over $\mathbf{X}_1 \setminus \{\xa\}$. Let $G_1 \in E_1$. Let $\Pa$ be a parent set of size $k$ for $\xa$ that maximizes the family score. Formally,  
		\[
		\Pa \in \argmax_{\Pi\subseteq \mathbf{X}_1\setminus \xa:|\Pi|=k}\score(\{\family{\xa,\Pi}\}).
		\]
		Denote $f_a^* := \family{\xa, \Pa}$. Let $G$ be an optimal graph over $\mathbf{X}_1$, and let $f_a$ be the family of $\xa$ in $G$. By assumption \ref{xdchildless} on $\cD_1$, $\xa$ has no children in $G$. Therefore, $G \setminus \{f_a\}$ is a DAG with an in-degree at most $k$ over $\mathbf{X}_1 \setminus \{\xa\}$. Thus, $\score(G \setminus \{f_a\}) \leq \score(G_1)$. It follows that
		\[
		\score(G) = \score(G \setminus \{f_a\}) + \score(\{f_a\}) \leq \score(G \setminus \{f_a\}) + \score(\{f_a^*\}) \leq \score(G_1) + \score(\{f_a^*\}) = \score(G_1 \cup \{f_a^*\}).
		\]
		Therefore, $G_1 \cup \{f_a^*\}$ is also optimal on $\mathbf{X}_1$.
		Now, assume for contradiction that there is another optimal EC over $\mathbf{X}_1\setminus \{\xa\}$, denoted $E_2$, and let $G_2 \in E_2$. By the same analysis as above, $G_2 \cup \{f_a^*\}$ is also optimal on $\mathbf{X}_1$.
		However, by \thmref{vstructures}, since $G_1$ and $G_2$ are not equivalent, then $G_1 \cup \{f_a^*\}$ and $G_2 \cup \{f_a^*\}$ are also not equivalent, contradicting assumption \ref{gap} on $\cD_1$. Therefore, $E_1$ is the only optimal EC over $\mathbf{X}_1 \setminus \{ X_a\}$. 
		
		Now, let $G_3$ be non-optimal DAG with an in-degree at most $k$ over $\mathbf{X}_1 \setminus \{\xa\}$, so $\score(G_3) < \score(G_1)$. By assumption \ref{gap} on $\cD_1$, 
		$\score(G_3  \cup \{f_a^*\}) \leq \score(G_1\cup \{f_a^*\}) - \beta$.
		Therefore, $\score(G_3) \leq \score(G_1) - \beta$. This proves the claim.
	\end{proof}

	\thmref{d2main} is now an immediate consequence of the lemmas above, since both properties of \defref{gammastable} hold for all positive $\gamma \leq \beta-3d\lambda$, by \lemref{optimalecgap} and \lemref{gapd1}.

	\section{Experiments}
	\label{sec:experiments}
	
	\newcommand{\fulltable}

{
  \begin{table}[h]
		\begin{center}
			\begin{tabular}{lllllll}
				\toprule
				$\mathbf{d}$ & $\mathbf{r}$ & $\mathbf{\epsilon\equiv d/r}$ & $\mathbf{N}$ \textbf{naive}  & $\mathbf{N}$ \textbf{active}  & \textbf{sample} & \textbf{\% accepted}\\
				 &  &  &  &   & \textbf{ratio}  &  \textbf{families}\\ 
				\toprule \multirow{4}{*}{6} 
				& $2^{13}$ &0.00073 & \textbf{43}$\times10^{12}$ &51$\times10^{12}$ &118\% &0\% \\  
				 & $2^{15}$ &0.00018 & \textbf{88}$\times10^{13}$ &104$\times10^{13}$ &118\% &0\% \\  
				 & $2^{17}$ &0.00005 & \textbf{17}$\times10^{15}$ &20$\times10^{15}$ &118\% &0\% \\  
				 & $2^{19}$ &0.00001 & \textbf{33}$\times10^{16}$ &39$\times10^{16}$ &117\% &0\% \\  
				\midrule\multirow{4}{*}{7} 
				& $2^{13}$ &0.00085 & 81$\times10^{12}$ &\textbf{73}$\pm10\times10^{12}$ &90\% &11\% $\pm$5.71\% \\ 
				& $2^{15}$ &0.00021 & 16$\times10^{14}$ &\textbf{13}$\times10^{14}$ &84\% &14\% \\  
				& $2^{17}$ &0.00005 & 32$\times10^{15}$ &\textbf{25}$\pm4\times10^{15}$ &79\% &17\% $\pm$8.57\% \\ 
				& $2^{19}$ &0.00001 & 62$\times10^{16}$ &\textbf{21}$\times10^{16}$ &34\% &42\% \\  
				\midrule\multirow{4}{*}{8} 
				& $2^{13}$ &0.00098 & \textbf{13}$\times10^{13}$ &15$\times10^{13}$ &116\% &0\% \\  
				& $2^{15}$ &0.00024 & 27$\times10^{14}$ &\textbf{24}$\times10^{14}$ &87\% &12\% \\  
				& $2^{17}$ &0.00006 & 54$\times10^{15}$ &\textbf{47}$\times10^{15}$ &87\% &12\% \\  
				& $2^{19}$ &0.00002 & 105$\times10^{16}$ &\textbf{26}$\times10^{16}$ &24\% &50\% \\  
				\midrule\multirow{4}{*}{9} 
				& $2^{13}$ &0.00110 & 21$\times10^{13}$ &\textbf{19}$\times10^{13}$ &91\% &11\% \\  
				& $2^{15}$ &0.00027 & 43$\times10^{14}$ &\textbf{21}$\times10^{14}$ &48\% &33\% \\  
				& $2^{17}$ &0.00007 & 85$\times10^{15}$ &\textbf{41}$\times10^{15}$ &48\% &33\% \\  
				& $2^{19}$ &0.00002 & 165$\times10^{16}$ &\textbf{30}$\times10^{16}$ &18\% &55\% \\  
				\midrule\multirow{4}{*}{10} 
				& $2^{13}$ &0.00122 & 31$\times10^{13}$ &\textbf{24}$\pm2\times10^{13}$ &76\% &18\% $\pm$4.00\% \\ 
				& $2^{15}$ &0.00031 & 64$\times10^{14}$ &\textbf{34}$\times10^{14}$ &54\% &30\% \\  
				& $2^{17}$ &0.00008 & 126$\times10^{15}$ &\textbf{68}$\times10^{15}$ &54\% &30\% \\  
				& $2^{19}$ &0.00002 & 244$\times10^{16}$ &\textbf{69}$\pm43\times10^{16}$ &28\% &49\% $\pm$13.00\% \\ 
				\midrule\multirow{4}{*}{11} 
				& $2^{13}$ &0.00134 & 45$\times10^{13}$ &\textbf{34}$\times10^{13}$ &75\% &18\% \\  
				& $2^{15}$ &0.00034 & 90$\times10^{14}$ &\textbf{53}$\times10^{14}$ &58\% &27\% \\  
				& $2^{17}$ &0.00008 & 17$\times10^{16}$ &\textbf{10}$\times10^{16}$ &58\% &27\% \\  
				& $2^{19}$ &0.00002 & 34$\times10^{17}$ &\textbf{20}$\times10^{17}$ &59\% &27\% \\  
				\midrule\multirow{6}{*}{12} 
				& $2^{13}$ &0.00146 & \textbf{61}$\times10^{13}$ &72$\times10^{13}$ &116\% &0\% \\  
				& $2^{15}$ &0.00037 & 124$\times10^{14}$ &\textbf{34}$\pm6\times10^{14}$ &27\% &47\% $\pm$3.82\% \\ 
				& $2^{17}$ &0.00009 & 245$\times10^{15}$ &\textbf{67}$\pm16\times10^{15}$ &27\% &51\% $\pm$3.33\% \\ 
				& $2^{19}$ &0.00002 & 475$\times10^{16}$ &\textbf{74}$\pm22\times10^{16}$ &15\% &58\% $\pm$5.03\% \\ 
				
				\bottomrule
				
			\end{tabular}

	\end{center}
	\caption{Experiment results for the synthetic networks. $N$ stands for the number of samples used by each algorithm. "sample ratio" is the sample size savings rate by \algname\ compared to the naive algorithm (N active / N naive).  ``\% accepted families'' is the fraction of the families that \algname\ accepted before the last stage.}\label{tab:exp_results}
\end{table} 

}

{
	\begin{table}[h]
		\begin{center}
			
			\begin{tabular}{llllllll}
				\toprule
				\textbf{Network} & $\mathbf{d}$ & $\mathbf{r}$ & $\mathbf{\epsilon\equiv d/r}$ & $\mathbf{N}$ \textbf{naive}  & $\mathbf{N}$ \textbf{active}  & \textbf{sample} & \textbf{\% accepted}\\
				
				& &  &  &  &   & \textbf{ratio} &  \textbf{families}\\ 
				\toprule \multirow{5}{*}{Cancer} 
				& \multirow{5}{*}{5} 
				& $2^{11}$ &0.00244 & \textbf{96}$\times10^{10}$ &116$\times10^{10}$ &120\% &0\% \\  
				& & $2^{13}$ &0.00061 & \textbf{20}$\times10^{12}$ &24$\times10^{12}$ &119\% &0\% \\  
				& & $2^{15}$ &0.00015 & 41$\times10^{13}$ &\textbf{30}$\times10^{13}$ &73\% &20\% \\  
				& & $2^{17}$ &0.00004 & 81$\times10^{14}$ &\textbf{13}$\times10^{14}$ &16\% &60\% \\  
				& & $2^{19}$ &0.00001 & 157$\times10^{15}$ &\textbf{24}$\times10^{15}$ &15\% &60\% \\  
				\midrule\multirow{5}{*}{Earthquake} 
				& \multirow{5}{*}{5} 
				& $2^{11}$ &0.00244 & \textbf{96}$\times10^{10}$ &116$\times10^{10}$ &120\% &0\% \\  
				& & $2^{13}$ &0.00061 & \textbf{20}$\times10^{12}$ &24$\times10^{12}$ &119\% &0\% \\  
				& & $2^{15}$ &0.00015 & \textbf{41}$\times10^{13}$ &49$\times10^{13}$ &119\% &0\% \\  
				& & $2^{17}$ &0.00004 & 81$\times10^{14}$ &\textbf{21}$\pm9\times10^{14}$ &25\% &52\% $\pm$9.80\% \\ 
				& & $2^{19}$ &0.00001 & 157$\times10^{15}$ &\textbf{47}$\pm18\times10^{15}$ &30\% &48\% $\pm$9.80\% \\ 
				\midrule\multirow{5}{*}{Survey} 
				& \multirow{5}{*}{6} 
				& $2^{11}$ &0.00293 & \textbf{20}$\times10^{11}$ &24$\times10^{11}$ &118\% &0\% \\  
				& & $2^{13}$ &0.00073 & 43$\times10^{12}$ &\textbf{21}$\times10^{12}$ &48\% &33\% \\  
				& & $2^{15}$ &0.00018 & 88$\times10^{13}$ &\textbf{22}$\times10^{13}$ &25\% &50\% \\  
				& & $2^{17}$ &0.00005 & 174$\times10^{14}$ &\textbf{44}$\times10^{14}$ &25\% &50\% \\  
				& & $2^{19}$ &0.00001 & 338$\times10^{15}$ &\textbf{86}$\times10^{15}$ &25\% &50\% \\  
				\midrule\multirow{5}{*}{Asia} 
				& \multirow{5}{*}{8} 
				& $2^{11}$ &0.00391 & \textbf{65}$\times10^{11}$ &76$\times10^{11}$ &116\% &0\% \\  
				& & $2^{13}$ &0.00098 & \textbf{13}$\times10^{13}$ &15$\times10^{13}$ &116\% &0\% \\  
				& & $2^{15}$ &0.00024 & 27$\times10^{14}$ &\textbf{24}$\times10^{14}$ &87\% &12\% \\  
				& & $2^{17}$ &0.00006 & 545$\times10^{14}$ &\textbf{64}$\pm28\times10^{14}$ &11\% &65\% $\pm$7.50\% \\ 
				& & $2^{19}$ &0.00002 & 1055$\times10^{15}$ &\textbf{94}$\pm39\times10^{15}$ &8\% &67\% $\pm$6.12\% \\ 
				\midrule\multirow{5}{*}{Sachs} 
				& \multirow{5}{*}{11} 
				& $2^{11}$ &0.00537 & \textbf{48}$\times10^{12}$ &54$\pm6\times10^{12}$ &111\% &1\% $\pm$5.45\% \\ 
				& & $2^{13}$ &0.00134 & \textbf{10}$\times10^{14}$ &11$\pm1\times10^{14}$ &109\% &2\% $\pm$5.82\% \\ 
				& & $2^{15}$ &0.00034 & 204$\times10^{14}$ &\textbf{57}$\pm31\times10^{14}$ &28\% &59\% $\pm$22.73\% \\ 
				& & $2^{17}$ &0.00008 & 402$\times10^{15}$ &\textbf{25}$\pm15\times10^{15}$ &6\% &75\% $\pm$4.17\% \\ 
				& & $2^{19}$ &0.00002 & 776$\times10^{16}$ &\textbf{49}$\pm14\times10^{16}$ &6\% &71\% $\pm$2.73\% \\

				\bottomrule
				
			\end{tabular}
		
		\end{center}
		\caption{Experiment results for benchmark networks. $N$ stands for the number of samples used by each algorithm. "sample ratio" is the sample size savings rate by \algname\ compared to the naive algorithm (N active / N naive). ``\% accepted families'' is the fraction of the families that \algname\ accepted before the last stage.}\label{tab:benchmark_results}
	\end{table} 
	
}

In this section, we report an empirical comparison between the naive algorithm, given in \secref{naive}, and \algname, given in \secref{newalg}. The code for both algorithms and for the experiments below is available at \codeurl.

We implemented the algorithms for the case of discrete distributions using the plug-in estimator for conditional entropy to calculate the empirical score $\escore$ (see \eqref{empscore}), as discussed in \secref{naive}. To avoid spurious score ties, we augmented the empirical score with an additional tie-breaking term. 
Since the empirical entropy never increases when adding a potential parent to a family, any family with a maximal score that has fewer than $k$ parents can be augmented with additional parents without decreasing the score. This creates spurious optimal families, thus significantly reducing the number of cases where an optimal family can be identified and accepted by \algname. To break the score ties, we added to the entropy of each family a small penalty term that prefers smaller families.
Formally, we used the following modified empirical score:
\begin{equation}
\escore(G):= -\sum_{i\in [d]}\left(\hat{H}(X_i \mid \Pi_i(G)) + \beta|\Pi_i(G)|\cdot\hat{H}(X_i)\right),
\end{equation}
where $\beta := 0.001$.

We now describe how we implemented the non-trivial steps in each algorithm. The naive algorithm aims to output a network structure in $\struct_{d,k}$ that maximizes the empirical score, as defined  above. Finding a structure with the maximal score is computationally hard in the general case, as discussed in \secref{intro}. We use the well-established algorithm GOBNILP \citep{Cussens11}, as implemented in the eBNSL package\footnote{\url{https://github.com/alisterl/eBNSL}} \citep{LiaoShCuVa18}, which attempts to find such a structure using a linear programming approach.

We now turn to the implementation of \algname. To compute $N_t$, the number of required samples (see line \ref{Nt}), we used the formula given in the bound in \lemref{N}. To compute $L_j$, the set of possible equivalence classes  (see line \ref{L_j}), we used the eBNSL algorithm, also implemented in the eBNSL package. The input to eBNSL includes the score of each family and a positive real number $\epsilon$. The algorithm heuristically attempts to output all network structures in $\struct_{d,k}$ with a score gap of at most $\epsilon$ from the maximal achievable score. To compute $L_j$, \algname\ sets the $\epsilon$ input value of eBNSL to $\theta_j$. The list of structures provided as output by eBNSL is then divided into equivalence classes, using the characterization given in \thmref{vstructures}. To impose the constraint that $L_j$ should only include structures consistent with the accepted families, as required by the definition of $L_j$, \algname\ provides eBNSL with structural constraints as additional input. These constraints specify sets of edges that must or must not be included in the output structures. To find a a structure that maximizes the empirical score, \algname\ uses GOBNILP, providing it with constraints that impose the inclusion of accepted families.

We ran experiments using data generated from synthetically constructed networks, and from discrete networks from the Bayesian Network Repository,\footnote{\url{https://www.bnlearn.com/bnrepository/}} which provides Bayesian networks that are commonly used as benchmarks.  We generated synthetic networks with 6-12 nodes, and tested all benchmark networks from the ``small networks'' category of the repository, which includes networks with up to 11 nodes. In all the networks, all nodes represented Bernoulli random variables taking values in $\{0,1\}$, except for the benchmark networks Survey and Sachs, which include multinomially distributed variables with $3$ possible values.

The synthetic networks describe stable distributions of the type presented in \secref{stableex}, in which each of the network variables has at most  two parents. The nodes in a network with $d$ nodes, denoted $\cB_d$, are denoted: $X_1,\ldots, X_{d-2}, \xa, \xb$.
\figref{gammav} illustrates the graph structures for $\cB_6$ and $\cB_7$. In all networks, $X_1$ has no parents and is a Bernoulli(1-$\rho$) random variable, where $\rho = 0.99$. $X_2$ has only $X_1$ as a parent and is equal to $X_1$ with an independent probability of $1-\rho$ (otherwise, it is equal to $1-X_1$). Each of the variables $X_i$ for $i \in \{3,...,d-2\}$ has two parents $X_j, X_{j'}$ for some $j,j' < i$.
The value of $X_i$ is $\mathrm{xor}
(X_j, X_{j'})$ with an independent probability of $\rho$. In all networks, $\xa$ is the child of $X_3$ and $X_4$, and the value of $\xa$ is $\mathrm{xor}
(X_3, X_{4})$ with an independent probability of $\rho$. $\xb$ is a slightly noisy version of $\xa$, such that $\xa=\xb$ with an independent probability of $1-5\cdot10^{-6}$. This construction satisfies the conditions in \secref{stableex} with $\mathbf{X}_1:=\{X_1,\ldots,X_{d-2},\xa\}$. Since $\xb$ is almost always equal to $\xa$, it is hard to distinguish between them using samples. In particular, the structure in which $X_3,X_4$ are the parents of $\xb$ and the latter is the parent of $\xa$ has a score which is very close to optimal. The advantage of the active algorithm is in being able to focus on observations that distinguish between $\xa$ and $\xb$.

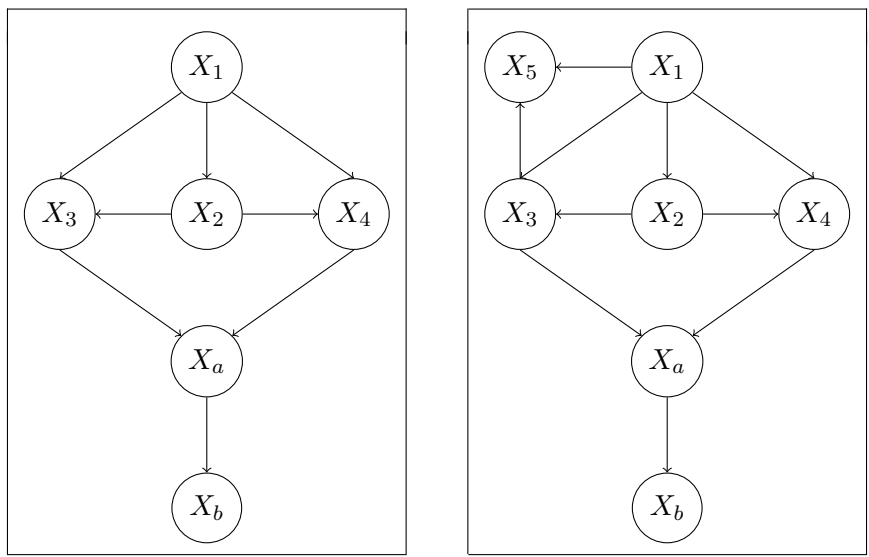
\begin{figure}[h!]
	\begin{center}
		\begin{tabular}[ht]{| c | c |c|}
			\cline{1-1}\cline{3-3}
			& & \\[-5pt]
			\begin{tikzpicture}[
			observednode/.style={circle, draw},
			hiddennode/.style={circle, draw, dashed},
			]
			\node[observednode] (x0) {$X_1$};
			\node[observednode, below=of x0] (x1) {$X_2$};
			\node[observednode, left=of x1] (x2) {$X_3$};
			\node[observednode, right=of x1] (x3) {$X_4$};
			\node[observednode, below=of x1] (x4) {$\xa$};
			\node[observednode, below=of x4] (x5) {$\xb$};

			\draw[->] (x0.south) -- (x1.north);
			\draw[->] (x0.south west) -- (x2.north);
			\draw[->] (x0.south east) -- (x3.north);
			\draw[->] (x1.west) -- (x2.east);
			\draw[->] (x1.east) -- (x3.west);
			\draw[->] (x2.south) -- (x4.north west);
			\draw[->] (x3.south) -- (x4.north east);
			\draw[->] (x4.south) -- (x5.north);

			\end{tikzpicture} & $\quad$ &
			\begin{tikzpicture}[
			observednode/.style={circle, draw},
			hiddennode/.style={circle, draw, dashed},
			]
			\node[observednode] (x0) {$X_1$};
			\node[observednode, below=of x0] (x1) {$X_2$};
			\node[observednode, left=of x1] (x2) {$X_3$};
			\node[observednode, right=of x1] (x3) {$X_4$};
			\node[observednode, below=of x1] (x4) {$\xa$};
			\node[observednode, below=of x4] (x5) {$\xb$};
			
			\node[observednode, above=of x2] (x6) {$X_5$};
			
			\draw[->] (x0.south) -- (x1.north);
			\draw[->] (x0.south west) -- (x2.north);
			\draw[->] (x0.south east) -- (x3.north);
			\draw[->] (x1.west) -- (x2.east);
			\draw[->] (x1.east) -- (x3.west);
			\draw[->] (x2.south) -- (x4.north west);
			\draw[->] (x3.south) -- (x4.north east);
			\draw[->] (x4.south) -- (x5.north);
			
			\draw[->] (x0.west) -- (x6.east);
			\draw[->] (x2.north) -- (x6.south);
			
			\end{tikzpicture}\\
			\cline{1-1}\cline{3-3}
		\end{tabular}
		
		\caption{Illustration of the networks $\cB_6$ and $\cB_7$ used in the experiments.}\label{fig:gammav}
	\end{center}
\end{figure}

In all of the synthetic and benchmark networks, the true network could be represented with $k=2$, except for the benchmark Sachs network, in which the true network requires $k=3$. Accordingly, the value of $k$ was set to $2$ for all but the Sachs network, where it was set to $3$. In all of the experiments, we set $\delta=0.05$.

To set the values of $\epsilon$ in each experiment in a comparable way, we adjusted this value by the number of network nodes, denoted by $d$, by
fixing the ratio $r:=d / \epsilon$ and calculating $\epsilon$ based on this ratio. We ran experiments with $r := 2^j$ for a range of values of $j$.
Each experiment configuration was ran $10$ times, and we report the average and standard deviation of the results. In all of the experiments, both the naive algorithm and \algname\ found an $\epsilon$-optimal structure.  Note that for the naive algorithm, all repeated experiments use the same sample size, since this size is pre-calculated.

Detailed results of the synthetic and the benchmark experiments are provided in \tabref{exp_results} and \tabref{benchmark_results}, respectively. The ``sample ratio'' column calculates the average reduction in the number of samples when using \algname\ instead of the naive algorithm, where a number below $100\%$ indicates a reduction in the number of samples. The last column lists the percentage of families that were accepted early in each experiment.
\fulltable

It can be seen in both tables that \algname\ obtains a greater reduction for larger values of $r$ (smaller values of $\epsilon$). 
This is also depicted visually in \figref{changing_d}, which plots the sample sizes obtained by the algorithms for the synthetic networks as a function of $d$ for three values of $r$. It can be seen that for $r=2^{13}$, there is almost no difference between the required sample sizes, while some sample saving is observed for $r=2^{15}$, and a large saving is observed for $r=2^{19}$. For the two greater values of $r$, it is apparent that the advantage increases with $d$, showing that larger values of $d$ lead to a larger sample size reduction.
In the most favorable configurations for both the synthetic and the benchmark networks, when $r=2^{19}$, and the largest tested value of $d$, \algname\ samples only 15\% and 6\%, respectively, of the number of samples required by the naive algorithm.

The increase in sample size reduction when increasing $r$ can be explained by observing in both tables that when $r$ is small, only a small fraction of the families is accepted early by \algname, leading to both algorithms requiring about the same sample size. In general, \algname\ is more successful when it accepts a larger fraction of the families early.

\begin{figure}[p!]
	
	\centering
			\includegraphics[scale=0.9]{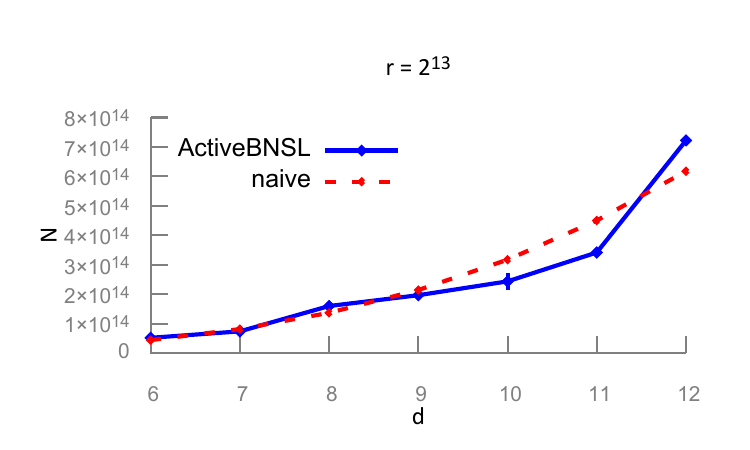}\\
			\includegraphics[scale=0.9]{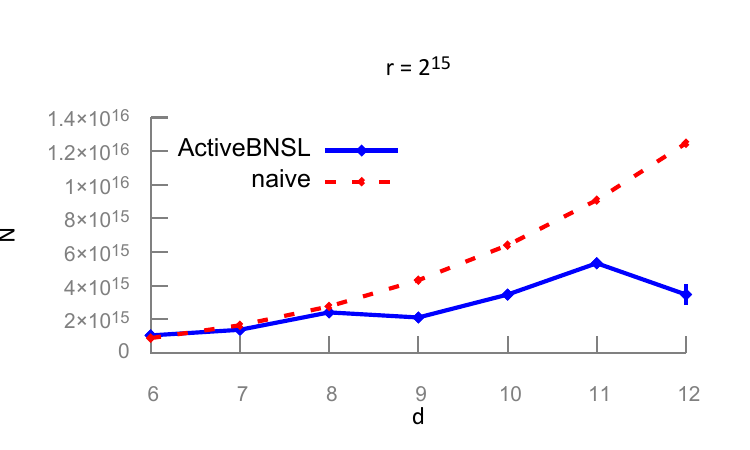}\\
			\includegraphics[scale=0.9]{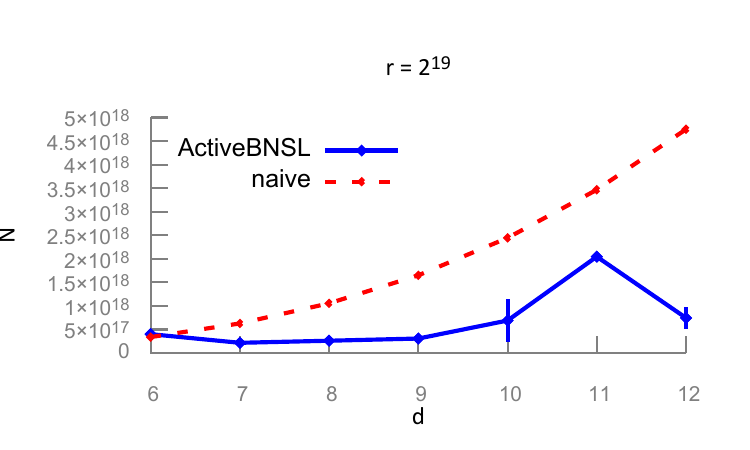}
	
                        \caption{Number of samples taken by \algname\ (active) and the naive algorithm, $N$, as a function of $d$, the number of nodes in the network, for three values of $r$. 
                        }
\label{fig:changing_d}
\end{figure}

To summarize, the reported experiments demonstrate that despite the computational hardness of \algname\ in the general case, it can be successfully implemented in practice using established approaches, and obtain a substantial improvement in the number of required samples. We note that the number of samples required for both the naive algorithm and \algname\ is quite large in our experiments. This is a consequence of the specific estimator and concentration bounds that we use, which may be loose in some cases. We expect that this number can be significantly reduced by using tighter concentration bounds tailored to this task. We leave this challenge for future work.

	\section{Discussion}\label{sec:discussion}
	
	This is the first work to study active structure learning for Bayesian networks in an observational setting. The proposed algorithm, \algname, can have a significantly improved sample complexity in applicable cases, and incurs a negligible  sample complexity penalty in other cases. 
	
	We considered the case where $k$, the maximal number of parents, is smaller by exactly one than the number of variable observations that can be made in a single sample. We now briefly discuss the implications of other cases. Denote by $l$ the number of variables that can be simultaneously observed, as determined based on external constraints (for instance, $l$ may indicate the maximal number of tests that each patient in an experiment can undergo). 
	As mentioned in \secref{related}, if $k > l-1$ then the score cannot be computed directly. Thus, if there are no other limitations on $k$, one may set $k := l-1$ to obtain the least constrained problem. Nonetheless, in some cases this value of $k$ may be too large, especially if $l$ is large, for instance due to computational reasons. This raises an interesting question: How do the results above generalize for the case of $k < l-1$? In this case, a non-interactive naive algorithm may be defined by fixing a set of subsets of size $l$ that cover all the possible subsets of size $k+1$, and sampling each of those sufficiently for uniform convergence. A version of ActiveBNSL might then be suggested, which samples the same set of fixed subsets of size $l$, and stops sampling such a subset only if a family was accepted for each of the variables in the subset. Characterizing the improvement of the active algorithm over the naive one in this case would depend on the initial choice of fixed subsets and their relationship to the underlying distribution. We leave a full analysis of this scenario for future work.
	
	We characterized a family of distributions in which \algname\ can obtain a significant sample complexity reduction, and demonstrated the reduction in experiments. A full characterization of distributions in which a sample complexity reduction is possible is an interesting open problem for future work.

	\bibliography{mybib}

	\appendix

	\clearpage
	\section{Convergence of the conditional entropy plug-in estimator}\label{ap:N}

	In this section, we prove \lemref{N}, which gives a bound on the number of random samples required to estimate the conditional entropy of a discrete random variable. We first bound the bias of the conditional-entropy plug-in estimator. 
	\cite{Paninski03} showed that the bias of the unconditional entropy estimator, for any random variable $A$ with support size $M_a$, satisfies
	\begin{equation}\label{eq:paninski}
	-\frac{M_a-1}{N} < \E[\hat{H}(A)] - \E[H(A)] < 0.
	\end{equation}
	The following lemma derives an analog result for the conditional entropy. 
	\begin{lemma}\label{lem:cond_bias}
		Let $A,B$ be discrete random variables with support cardinalities $M_a$ and $M_b$, respectively. Let $\hat{H}(A \mid B)$ be the plug-in estimator of $H(A \mid B)$, based on $N$ i.i.d.~copies of $(A,B)$. Then
		\begin{equation*}
		-\frac{(M_a-1)M_b}{N}\leq\mathbb{E}[\hat{H}(A\mid B)]-H(A\mid B) \leq 0.
		\end{equation*}
	\end{lemma}
	
	\begin{proof}
		Let $\cB$ be the support of $B$. For any $b \in \cB$,
		denote $p_b := P(B=b)$ and let $N_b$ be the number of examples  in the sample with $B = b$. We have
		\begin{equation}\label{eq:estcond}
		\E[\hat{H}(A \mid B)] = \sum_{b \in \cB}\E\left[\frac{N_b}{N} \hat{H}(A\mid B=b)\right].
		\end{equation}
		Bounding a single term in the sum, we have
		\begin{align}
		\E\left[\frac{N_b}{N}\cdot \hat{H}(A\mid B=b)\right]&=
		\sum_{n \in [N]}\frac{n}{N} \cdot P(N_b = n)\cdot \E[\hat{H}(A\mid B=b) \,\big|\, N_b = n]\notag\\
		&\geq\sum_{n \in [N]}\frac{n}{N}\cdot P(N_b = n)\cdot (H(A\mid B=b)-(M_a-1)/n),\label{eq:paneq}\\
		&= p_b\cdot H(A\mid B=b) - (M_a-1)/N,\notag
		\end{align}
		where the inequality on line (\ref{eq:paneq}) follows from \eqref{paninski}, and the last equality follows since 
		\[
		\sum_{n \in [N]}\frac{n}{N}P(N_b = n)=\E[N_b/N]=p_b,
		\]
		and
		\[
		\sum_{n \in [N]} \frac{n}{N} P(N_b = n) (M_a-1)/n = (M_a-1)/N.
		\]
		Plugging this into \eqref{estcond}, we get
		\begin{align*}
		\E[\hat{H}(A \mid B)] &= \sum_{b \in \cB} \E\left[\frac{N_b}{N}\cdot \hat{H}(A\mid B=b)\right]\\
		&\geq  \sum_{b \in \cB} (p_b \cdot H(A \mid B =b) - (M_a - 1)/N) \\
		&= H(A \mid B) - (M_a - 1)M_b/N.
		\end{align*}
		Proving that the bias is non-positive is done in a similar way.
		This completes the proof.
	\end{proof}

	We now use the lemma above to prove \lemref{N}.
	
	\begin{proof}[Proof of \lemref{N}]
		By McDiarmid's inequality \citep{McDiarmid89}, if the maximal absolute change that a single random sample can induce on the plug-in estimator is $\theta$, then
		\[
		P(|\hat{H}(A \mid B)-\mathbb{E}[\hat{H}(A \mid B)]| > \epsilon) \leq 2\exp(-2\epsilon^2/(N\theta^2)).
		\]
		Let $\cB$ be the support of $B$. For any $b \in \cB$,
		denote $p_b := P(B=b)$ and let $\hat{p}_b$ be the empirical fraction of samples with $B = b$. We have
		\[
		\hat{H}(A \mid B) = \sum_{b \in \cB} \hat{p}_b \cdot \hat{H}(A \mid B = b),
		\]
		where $\cB$ is the support of $B$ and 
		By \cite{ShamirSaTi10}, the maximal change that can be induced on $\hat{H}(A \mid B = b)$ when replacing an example $(a,b)$ in the sample with an example $(a',b)$ is $\log(\hat{p}_bN)/(\hat{p}_bN)$. Therefore, this type of replacement induces a change of at most $\log(\hat{p}_bN)/N \leq \log(N)/N$ on $\hat{H}(A \mid B)$. Now, if an example $(a,b)$ is replaced with an example $(a',b')$,  where $b\neq b'$, then the maximal change it induces on $\hat{H}(A \mid B)$ is $\log(M_a)/N$, since for any $b$, $\hat{H}(A \mid B = b) \in [0, \log(M_a)]$. Therefore, if $N \geq M_a$, both types of sample replacements induce a change of at most $\log(N)/N$ on $\hat{H}(A \mid B)$. It follows that
		\begin{equation}\label{eq:entropy_converegnce}
		P(|\hat{H}(A \mid B)-\E[\hat{H}(A \mid B)]|>\epsilon)\leq 2\exp(-\frac{2N\epsilon^2}{\log^2(N)}).
		\end{equation}
		
		If $\epsilon \geq (M_a-1)M_b/N$, then
		\begin{equation}\label{eq:conv2}
		P(|\hat{H}(A \mid B)-H(A \mid B)|>2\epsilon)\leq 2\exp(-\frac{2N\epsilon^2}{\log^2(N)}).
		\end{equation}
		Lastly, we derive a lower bound for $N$ such that the RHS is at most $\delta$.
		Let $\alpha \geq 1$  and assume $N \geq e^2$. We start by showing that if $N \geq \alpha\log^2(\alpha)$, then $N/\log^2(N) \geq \alpha/4$. 
		$N/\log^2(N)$ is monotonic increasing for $N \geq e^2$. 
		Therefore, for $N \geq \alpha\log^2(\alpha)$,
		\begin{align*}
		\frac{N}{\log^2(N)} &\geq \frac{\alpha\log^2(  \alpha)}{\log^2( \alpha\log^2(  \alpha))} = \frac{\alpha\log^2(\alpha)}{(\log(\alpha) + \log(\log^2( \alpha)))^2} 
		\end{align*}
		Since $\alpha \geq 1$, we have $\log^2(\alpha) \leq \alpha$. It follows that $\log(\log^2(\alpha)) \leq \log(\alpha)$. Therefore,
		\[
		(\log(\alpha) + \log(\log^2( \alpha)))^2 \leq (\log(\alpha) + \log(\alpha))^2  = 4\log^2(\alpha).
		\] 
		We conclude that for $N \geq \alpha\log^2(\alpha)$, $N/\log^2(N) \geq \alpha/4.$ Setting $\alpha := 2\log(2/\delta)/\epsilon^2$, we obtain that if $N \geq \alpha \log^2 \alpha$, then $N/\log^2(N) \geq \log(2/\delta)/(2\epsilon^2),$ which implies that the RHS of \eqref{conv2} is at most $\delta$. 
	\end{proof}

	\section{Proofs for technical lemmas in \secref{stableex}}\label{ap:deferred}
	In this section, we provide the proofs of technical lemmas stated in \secref{stableex}. 
	\begin{proof}[of \lemref{help}]
		To prove \eqref{xmidpi}, note that
		\begin{align*}
		H(\xa,\xb\mid\Pi)&=H(\xa,\xb,\Pi)-H(\Pi) =H(\xb\mid \xa,\Pi)+H(\xa,\Pi)-H(\Pi)\\
		&= H(\xb\mid \xa)+H(\xa\mid\Pi).
		\end{align*} 
		The last equality follows since by assumption \ref{coin}, $\Pi\rightarrow \xa\rightarrow \xb$ is a Markov chain, so that $H(\xb\mid \xa,\Pi) = H(\xb\mid \xa)$. From the above it we have
		\[
		H(\xa\mid \Pi) = H(\xa,\xb\mid\Pi)-H(\xb\mid \xa).
		\]
		It follows that
		\begin{align*}
		|H(\xb\mid\Pi)-H(\xa\mid\Pi)|&=|H(\xb\mid\Pi)-H(\xa,\xb\mid\Pi)+H(\xb\mid \xa)|\\
		&=|H(\xb,\Pi)-H(\xa,\xb,\Pi)+H(\xb\mid \xa)|\\
		&=|H(\xb\mid \xa)-H(\xa\mid \xb,\Pi)|\\
		&\leq \max(H(\xb\mid \xa),H(\xa\mid \xb,\Pi))\\
		&\leq \max(H(\xb\mid \xa), H(\xa\mid \xb))\leq \lambda.
		\end{align*}
		The last inequality follows from the definition of $\lambda$. This proves \eqref{xmidpi}.
		
		To prove \eqref{y}, note that
		\begin{align*}
		H(Y,\xb\mid\Pi,\xa)&=H(Y,\xb,\Pi,\xa)-H(\Pi,\xa)\\
		&=H(\xb\mid\Pi,Y,\xa)+H(\Pi,Y,\xa)-H(\Pi,\xa)\\
		&= H(\xb\mid \xa)+H(Y\mid\Pi,\xa),
		\end{align*} 
		where the last equality follows since by assumption \ref{coin}, $(Y,\Pi)\rightarrow \xa\rightarrow \xb$ is a Markov chain, so that $H(\xb\mid\Pi,Y,\xa)=H(\xb\mid \xa)$. The equality above implies that
		\[
		H(Y\mid\Pi,\xa)=H(Y,\xb\mid\Pi,\xa)-H(\xb\mid \xa).
		\]
		It follows that
		\begin{align*}
		&|H(Y\mid\Pi,\xb)-H(Y\mid\Pi,\xa)|\\& =|H(Y,\Pi,\xb)-H(\Pi,\xb)-H(Y,\xb\mid\Pi,\xa)+H(\xb\mid \xa)|\\
		&=|H(Y,\Pi,\xb)-H(\Pi,\xb)-H(Y,\xb,\Pi,\xa)+H(\Pi,\xa)+H(\xb\mid \xa)|\\
		&=|-H(\xa\mid Y,\Pi,\xb)-H(\xb\mid\Pi)+H(\xa\mid \Pi) +H(\xb\mid \xa)|\\
		&\leq H(\xa\mid Y,\Pi,\xb)+ |H(\xb\mid\Pi)-H(\xa\mid \Pi)| + H(\xb\mid \xa)\\
		&\leq H(\xa\mid \xb) + |H(\xb\mid\Pi)-H(\xa\mid \Pi)| + H(\xb\mid \xa) \leq 3\lambda.
		\end{align*}
		The last inequality follows from \eqref{xmidpi} and the definition of $\lambda$. This proves \eqref{y}.
	\end{proof}
	
	\begin{proof}[of \lemref{no_children}]
		If $\xb$ has no children in $G$, then set $\tilde{G} := G$. It is easy to see that in this case, $\tilde{G}$ satisfies all the required properties.
		
		Otherwise, denote the set of children of $\xb$ in $G$ by $W$. Denote by $\Pb,\Pa$ the parent sets of $\xb,\xa$ in $G$, respectively, and by $\Pi_w$ the parents of $X_w \in W$. Denote by $f_v:=\family{X_v,\Pi_v} \in G$ the family of $X_v$ in $G$. We construct $\tilde{G}$ from $G$ by replacing the families $f_v$,  for $\xa,\xb$ and all $X_w \in W$, by new families $\tilde{f}_v := \family{X_v, \tilde{\Pi}_v}$, as follows.
		
		First, define $\tilde{\Pi}_a$ and $\tilde{\Pi}_b$:
		\begin{enumerate}
			\item If $\xb$ has no directed path to $\xa$, define $\tilde{\Pi}_a := \Pi_a$ and $\tilde{\Pi}_b := \Pi_b$.
			\item Otherwise ($\xb$ has a directed path to $\xa$), 
			if $\xb$ is a parent of $\xa$, make $\xa$ a parent of $\xb$ instead. 
			In any case, switch between the other parents of $\xa$ and $\xb$. 
			Formally, set $\tilde{\Pi}_a = \Pb$. If $\xb\in\Pa$, $\tilde{\Pi}_b:=\Pa\setminus\{\xb\}\cup\{\xa\}$. Otherwise, $\tilde{\Pi}_b:=\Pa$. 
		\end{enumerate}
		\label{compliacted}
		Next, for each $w\in W \setminus \{\xa\}$, make $\xa$ a parent of $w$ instead of $\xb$. Formally, $\tilde{\Pi}_w:=\Pi_w\setminus\{\xb\}\cup\{\xa\}$.
		This completes the definition of $\tilde{G}$.
		
		It is easy to see that the maximal in-degree of $\tilde{G}$ is at most that of $G$, which is at most $k$, as required. 
		We now show that $\tilde{G}$ is a DAG, thus proving that $\tilde{G} \in \struct_{k,d}$.
		
		\begin{claim}
			$\tilde{G}$ is a DAG.
		\end{claim}
		\begin{claimproof}
			Consider first the case where $\xb$ does not have a directed path to $\xa$. Let $(A,B)$ be an edge in $\tilde{G}$ that is not in $G$. Then $B \in W$ and $A = \xa$.
			In $G$, there is no directed path from $B$ to $\xa$, since $B$ is a child of $\xb$ and this would create a directed path from $\xb$ to $\xa$. But all paths from $B$ in $\tilde{G}$ are the same as in $G$. Thus, there is no cycle in $\tilde{G}$.
			
			Now, consider the case where $\xb$ has a directed path to $\xa$. 
			Let $(A,B)$ be an edge in $\tilde{G}$ that is not in $G$. We will show that it does not create a cycle in $\tilde{G}$. Divide into cases:
			\begin{enumerate}
				\item If $B=\xb$, then $(A,B)$ is not an edge in a cycle in $\tilde{G}$, since $\xb$ has no children in $\tilde{G}$.	
				
				\item If $B\in W$, then $A=\xa$, since $\xa$ is the only parent added to $W$. Thus, $(A,B)$ is part of a cycle in $\tilde{G}$ only if $B$ is an ancestor of $\xa$ in $\tilde{G}$. Assume in contradiction that this is the case, and let $v_1,\ldots,v_l$ be the nodes in a shortest directed path from $B$ to $\xa$. $v_l$ is a parent of $\xa$ in $\tilde{G}$, thus it is a parent of $\xb$ in $G$. The nodes $v_1,\ldots, v_l$ cannot be $\xa$ or $\xb$ (because it has no children in $\tilde{G}$), thus their parents are the same in $G$ and $\tilde{G}$. It follows that the directed path $B,v_1,\ldots,v_l,\xb$ exists in $G$. However, $B$ is a child of $\xb$ in $G$, thus this implies a cycle in $G$, a contradiction. 
				Therefore, in this case $(A,B)$ is not part of a cycle.
				
				\item If $B=\xa$, then $A\in\Pb$. Thus, $(A,B)$ is part of a cycle if there is a directed path from $\xa$ to $A$ in $\tilde{G}$. Assume in contradiction that this is the case, and let $v_1,\ldots,v_l$ be the nodes in a shortest directed path from $\xa$ to $A$ in $\tilde{G}$. The nodes $v_2,\ldots, v_l$ and their parents cannot be $\xa$ or $\xb$, thus their parents are the same in $G$ and $\tilde{G}$. It follows that the directed path $v_1,v_2,\ldots,v_l,A$ exists in $G$. Now, if $v_1$ is a child of $\xb$ in $G$, this means that there is a directed path in $G$ from $\xb$ to its parent $A$, which means that there is a cycle in $G$.
				If $v_1$ is a child of $\xa$ in $G$, then, since $\xb$ has a directed path to $\xa$ in $G$, this again means that there is a path from $\xb$ to $A$ in $G$, leading to a cycle in $G$. In both cases, this contradicts the fact that $G$ is a DAG. 
				Thus, in this case as well, $(A,B)$ is not part of a cycle.
			\end{enumerate}
			
			Therefore, in this case as well, $\tilde{G}$ is a DAG.
			This completes the proof of the claim.
		\end{claimproof}
		
		To show that $\tilde{G}$ satisfies the properties stated in the lemma, observe first that properties (\ref{childless}) and (\ref{xa_child}) are immediate from the construction of $\tilde{G}$. Property (\ref{sim_score}) requires a bound on the difference in scores of $G$ and $\tilde{G}$. We have
		\[
		|\score(G) - \score(\tilde{G})| \leq \sum_{w \in W\setminus\{\xa\}}|H(f_w) - H(\tilde{f}_w)| + |H(f_a) - H(\tilde{f}_a) + H(f_b) - H(\tilde{f}_b)|.
		\]
		To bound the first term, let $\Pb' := \Pi_w \setminus \{\xb\}$. Then 
		\[
		|H(f_w) - H(\tilde{f}_w)| = |H(w \mid \Pb', \xb) - H(w \mid \Pb', \xa)| \leq 3\lambda,
		\]
		where the last inequality follows from \eqref{y}. Since there are at most $d-2$ such terms, the first term is at most $3(d-2)\lambda$.
		
		To bound the second term,
		first note that it is equal to zero if $\xb$ has no directed path to $\xa$ in $G$. If $\xb$ does have a directed path to $\xa$ in $G$, then let $\Pa' := \Pa \setminus \{\xb\}$.
		If $\xb$ is not a direct parent of $\xa$, then
		\begin{align*}
		&|H(f_b) - H(\tilde{f}_b) + H(f_a) - H(\tilde{f}_a)| \\
		&= |H(\xa\mid\Pa') + H(\xb\mid\Pb)-H(\xa\mid\Pb)- H(\xb\mid \Pa')| =: \Delta.
		\end{align*}
		If $\xb$ is a direct parent of $\xa$, then 
		\begin{align*}
		&|H(f_a) - H(\tilde{f}_a) + H(f_b) - H(\tilde{f}_b)| \\
		&= |H(\xa\mid\Pa',\xb) + H(\xb\mid\Pb)-H(\xa\mid\Pb)- H(\xb\mid\Pa',\xa)| \\
		&=
		|H(\xa\mid\Pa',\xb) + H(\xb\mid\Pb)-H(\xa\mid\Pb)- H(\xb\mid\Pa',\xa)|\\
		&=|H(\xb\mid\Pb)-H(\xa\mid\Pb)+H(\xa,\Pa')-H(\xb,\Pa')| = \Delta.
		\end{align*}
		Now, $\Delta$ can be bounded as follows: 
		\begin{align*}
		\Delta \leq |H(\xb\mid\Pb)-H(\xa\mid\Pb)|+|H(\xa\mid\Pa')-H(\xb\mid\Pa')|\leq 2\lambda,
		\end{align*}
		where the last inequality follows from \eqref{xmidpi}. Thus, the second term is upper bounded by $2\lambda$ in all cases. Property (\ref{sim_score}) is proved by summing the bounds of both terms.

	\end{proof}

	Lastly, we prove \thmref{scoreeps}.

	\begin{proof}[of \thmref{scoreeps}]
		Set $\lambda < \epsilon/2$. We prove that there exists a non-optimal EC with a score at least $\score^* -2\lambda$.
		Let $G^*\in\struct_{d,k}$ be an optimal structure for $\cD_2(\lambda)$. By \thmref{d2main},  $\cD_2(\lambda)$ is $(\beta - 3d\lambda,\mathbf{X}_1\setminus \{\xa\})$-stable. Therefore, by \lemref{optimalstruct}, the set of families of $\bar{V} \equiv \{\xb,\xa\}$ in $G^*$ is an optimal legal family set. Denote this set $F := \{f_a,f_b\}$. First, we prove that $\xa$ and $\xb$ must be in a parent-child relationship.
		
		\begin{claim}
			In $F$, either $\xa$ is a parent of $\xb$, or vice versa.	
		\end{claim}
		
		\begin{claimproof}
			Assume for contradiction that the claim does not hold. We define a new legal family set $F'$ for $\bar{V}$ and show that its score is larger than the score of $F$, thus contradicting the optimality of $F$. Denote the parent set of $\xa$ in $f_a$ by $\bar{\Pi}_a$.
			Define a new family $f'_a := \family{X_a, \Pi'_a}$, where $\Pi'_a$ is obtained by adding $\xb$ as a parent to $\bar{\Pi}_a$. In addition, if this increases the number of parents of $\xa$ over $k$, then one of the other elements in $\bar{\Pi}_a$ is removed in $\Pi'_a$.
			Let $F' = \{f'_a,f_b\}$. According to Property \ref{xdchildless} in the definition of $\cD_2$, $\xa$ has no children in $\mathbf{X}_1\setminus\{\xa\}$, therefore, $\xa$ has no children in $G^*$, and so adding the edge $(X_a, X_b)$ cannot create a cycle. Therefore, $F'$ is a legal family set. We have
			\begin{align*}
			\score(F') - \score(F) &= \score(\{\family{\xa, \Pa'}\})-\score(\{\family{\xa,\Pa}\})= H(\xa \mid \Pa) - H(\xa \mid \Pa') \\
			&\geq H(\xa \mid \mathbf{X}_1 \setminus \{\xa\}) - H(\xa \mid \xb) \geq \alpha - \lambda,
			\end{align*}
			where the last inequality follows from assumptions \ref{alpha} and \ref{lambda}. Now, since by definition, $\lambda < \alpha$, it follows $\score(F') > \score(F)$, which contradicts the optimality of $F$. 
			This completes the proof of the claim.
		\end{claimproof} 
		
		Now, define a new graph as follows: Let $\tilde{F}$ a family set which is the same as $F$, except that the roles of $\xb$ and $\xa$ are swapped. Define $\tilde{G} := G^* \setminus F \cup \tilde{F}$. Notice that by \thmref{d2main}, $\cD_2$ is $(\gamma,V)$-stable for $V=\mathbf{X}_1\setminus\{\xa\}$. Then, by Property \ref{vgap} in \defref{gammastable}, both $\xa$ and $\xb$ have no children in $V$. This means that in $G^*$, one of $\xa$ and $\xb$ has the other as a child, and the other has no children. Thus, swapping them cannot create a cycle, implying that $\tilde{G} \in \struct_{d,k}$ and that $\tilde{F}$ is a legal family set.
		
		We now show that $\score^*-2\lambda \leq \score(\tilde{G}) < \score^*$, thus proving the theorem. Since we have
		\[
		\score^* - \score(\tilde{G}) = \score(F)-\score(\tilde{F}),
		\]
		it suffices to show that $0 <  \score(F)-\score(\tilde{F}) \leq 2\lambda$.

		Let $\bar{\Pi}_a$, $\bar{\Pi}_b$ be the parent sets of $\xa$ and $\xb$ in $F$, respectively.
		Let $\Pa := \bar{\Pi}_a \setminus \{\xb\}$ and $\Pb := \bar{\Pi}_b \setminus \{\xa\}$. If $\xa$ is a parent of $\xb$ in $F$, then 
		\begin{align*}
		\score(F)-\score(\tilde{F}) &= H(\xa\mid \Pb,\xb) + H(\xb\mid \Pa) -H(\xa\mid \Pa) - H(\xb\mid\Pb,\xa).
		\end{align*}
		Since $H(\xb\mid \Pb,\xa) = H(\xb, \Pb,\xa) - H(\xa \mid \Pb) - H(\Pb)$, and symmetrically for \mbox{$H(\xa \mid \Pb, \xb)$}, it follows that
		\[
		H(\xa\mid \Pb,\xb)-H(\xb\mid\Pb,\xa) = H(\xa\mid \Pb) -H(\xb\mid \Pb).
		\]
		Therefore,
		\begin{align}
		\score(F)-\score(\tilde{F}) = H(\xa\mid \Pb) -H(\xb\mid \Pb) +  H(\xb\mid \Pa) - H(\xa\mid\Pa).\label{eq:pitag}
		\end{align}
		By a symmetric argument, \eqref{pitag} holds also if $\xb$ is a parent of $\xa$ in $F$.
		Combining \eqref{pitag} with \eqref{xmidpi} in \lemref{help}, we have that $|\score(F)-\score(\tilde{F})| \leq 2\lambda$.
		We have left to show that $\score(F) > \score(\tilde{F})$, which will complete the proof of the theorem.

		We have, for any set $\Pi \subseteq \mathbf{X}_1$,
		\[
		H(\xb, \xa \mid \Pi) = H(\xb \mid \xa, \Pi) + H(\xa \mid \Pi) = H(\xb \mid \xa) + H(\xa \mid \Pi),
		\]
		where the last inequality follows since by assumption \ref{coin}, $\Pi \rightarrow \xa \rightarrow \xb$ is a Markov chain. 
		Therefore,
		\[
		H(\xb \mid \Pi) \equiv H(\xb, \xa \mid \Pi) - H(\xa \mid \xb, \Pi)  = H(\xb \mid \xa) + H(\xa \mid \Pi) - H(\xa \mid \xb, \Pi).
		\]
		Using the above for $\Pi_a$ and $\Pi_b$, we get
		\[
		H(\xb \mid \Pa) - H(\xb \mid \Pb) = H(\xa \mid \Pa)-  H(\xa \mid \Pb) +H(\xa \mid \xb, \Pb) - H(\xa \mid \xb, \Pa).
		\]
		Combined with \eqref{pitag}, it follows that 
		\begin{equation}\label{eq:sf}
		\score(F)-\score(\tilde{F}) = H(\xa \mid \xb, \Pb) - H(\xa \mid \xb, \Pa).
		\end{equation}

		Now, in $F$, one of $\xb,\xa$ is the parent of the other, and in $\tilde{F}$, the other case holds. We show that in both cases, $\score(F)-\score(\tilde{F}) > 0$.
		\begin{enumerate}
			\item Suppose that $\xa$ is the parent of $\xb$. In this case, we can assume without loss of generality that $\Pb = \emptyset$ in $F$, since $H(\xb \mid \xa) = H(\xb \mid \xa, \Pb)$ for any set $\Pb \subseteq \mathbf{X}_1$, due to assumption \ref{coin}. 
			Therefore, from \eqref{sf},
			\begin{align*}
			\score(F)-\score(\tilde{F}) &= H(\xa \mid \xb) - H(\xa\mid \xb,\Pa)\\
			&= H(\xa, \xb) - H(\xb) - H(\xa, \xb,\Pa) + H(\xb, \Pa) \\
			&= H(\Pa \mid \xb) - H(\Pa \mid \xa, \xb) = H(\Pa \mid \xb) - H(\Pa \mid \xa).
			\end{align*}
			The last inequality follows since $\xb \rightarrow \xa \rightarrow \Pa$ is a Markov chain by assumption \ref{coin}. 
			Now, by assumption \ref{coin}, we have
			\begin{align}
			&H(\Pa \mid \xb) \geq H(\Pa \mid \xb, C)\\
			&= H(\Pa \mid \xa, C= 1)\P[C= 1] + H(\Pa \mid C=0)\P[C=0]\notag\\
			&= H(\Pa \mid \xa)\P[C= 1] + H(\Pa)\P[C=0].\label{eq:czero}
			\end{align}
			In addition, $H(\Pa \mid \xa) < H(\Pa)$. This is because $\Pa$ is an optimal parent set for $\xa$ from $\mathbf{X}_1$, and by assumption \ref{gap} on $\cD_1$, it is unique. Therefore, $H(\xa \mid \Pa) > H(\xa)$, which implies $H(\Pa \mid \xa) < H(\Pa)$. Combined with \eqref{czero}, and since $\P[C=0] > 0$ by assumption \ref{coin}, we have $H(\Pa \mid \xb) > H(\Pa \mid \xa)$. It follows that $\score(F)-\score(\tilde{F}) > 0$, as required.

			\item Suppose that $\xb$ is a parent of $\xa$ in $F$. Assume for contradiction that $\score(F) = \score(\tilde{F})$. Then $\tilde{F}$ is an optimal legal family set in which $\xa$ is a parent of $\xb$. Therefore, by the first case above, $\score(\tilde{F}) > \score(F)$, in contradiction to the optimality of $F$. 
		\end{enumerate}
		
		It follows that $\score(F) > \score(\tilde{F})$, which proves the claim.	
		
		Thus, $\tilde{G}$ is not an optimal graph, as required. This completes the proof of the theorem.
		
	\end{proof}

\end{document}